\documentclass[authoryear,12pt]{elsarticle}
\usepackage{amsmath,amssymb,amsthm}
\usepackage[ruled]{algorithm2e}
\usepackage{graphicx}
\usepackage{color}
\usepackage{xcolor}
\usepackage{booktabs}
\usepackage{caption}
\usepackage{subcaption}
\usepackage{soul}
\usepackage{xurl}
\graphicspath{{fig/}}

\usepackage{setspace}
\usepackage[T1]{fontenc}				
\usepackage[nomarkers]{endfloat}        
\usepackage{placeins}
\usepackage{ragged2e}

\usepackage{listings}
\usepackage{xcolor}
\definecolor{codegreen}{rgb}{0,0.6,0}
\definecolor{codegray}{rgb}{0.5,0.5,0.5}
\definecolor{codepurple}{rgb}{0.58,0,0.82}
\definecolor{backcolour}{rgb}{0.95,0.95,0.92}
\lstdefinestyle{mystyle}{
    backgroundcolor=\color{backcolour},   
    commentstyle=\color{codegreen},
    keywordstyle=\color{magenta},
    numberstyle=\tiny\color{codegray},
    stringstyle=\color{codepurple},
    basicstyle=\ttfamily\footnotesize,
    breakatwhitespace=false,         
    breaklines=true,                 
    captionpos=b,                    
    keepspaces=true,                 
    numbers=left,                    
    numbersep=5pt,                  
    showspaces=false,                
    showstringspaces=false,
    showtabs=false,                  
    tabsize=2
}
\allowdisplaybreaks

\renewcommand{\v}[1]{\boldsymbol{\mathbf{#1}}}
\newcommand{\vdot}[1]{\dot{\v{#1}}}
\newcommand{\vhat}[1]{\hat{\v{#1}}}
\newcommand{\vtilde}[1]{\widetilde{\v{#1}}}
\newcommand{\vbar}[1]{\bar{\v{#1}}}
\newcommand{\T}{\top}
\newcommand{\vt}[1]{\v{#1}^\T}
\newcommand{\R}[1][]{\mathbb{R}^{#1}}
\newcommand{\bmat}[1]{\begin{bmatrix} #1 \end{bmatrix}}
\newcommand{\norm}[1]{\|#1\|}
\newcommand{\dv}[2]{\frac{d #1}{d #2}}

\newcommand{\qty}[1]{\left(#1\right)}
\newcommand{\bqty}[1]{\left[#1\right]}
\newcommand{\ri}[1]{{(#1)}}

\newtheorem{theorem}{Theorem}
\newtheorem{lemma}{Lemma}
\newtheorem{remark}{Remark}

\journal{Elsevier}
\begin{document}
\baselineskip 24pt
\begin{frontmatter}
\author[1]{Zihao Wang}
\author[1]{Donghan Yu}
\author[1]{Zhe Wu\corref{a}}
\cortext[a]{Corresponding author. E-mail: wuzhe@nus.edu.sg.\\Please refer to \url{https://github.com/killingbear999/ICLSTM} for source codes.}

\address[1]{Department of Chemical and Biomolecular Engineering, National University of Singapore, 117585, Singapore}

\title{Real-Time Machine-Learning-Based Optimization Using Input Convex Long Short-Term Memory Network}

\begin{abstract}
Neural network-based optimization and control methods, often referred to as black-box approaches, are increasingly gaining attention in energy and manufacturing systems, particularly in situations where first-principles models are either unavailable or inaccurate. However, their non-convex nature significantly slows down the optimization and control processes, limiting their application in real-time decision-making processes. To address this challenge, we propose a novel Input Convex Long Short-Term Memory (IC-LSTM) network to enhance the computational efficiency of neural network-based optimization. Through two case studies employing real-time neural network-based optimization for optimizing energy and chemical systems, we demonstrate the superior performance of IC-LSTM-based optimization in terms of runtime. Specifically, in a real-time optimization problem of a real-world solar photovoltaic energy system at LHT Holdings in Singapore, IC-LSTM-based optimization achieved at least 4-fold speedup compared to conventional LSTM-based optimization. These results highlight the potential of IC-LSTM networks to significantly enhance the efficiency of neural network-based optimization and control in practical applications.
\end{abstract}
	
\end{frontmatter}

\section{Introduction}
Model-based optimization and control have been widely applied in energy and chemical systems in decades \citep{cai2009optimization, eisenhower2012methodology, wang2015operational, stadler2016model, lim2020optimal}. Traditional model-based optimization and control rely on the development of first-principles models, a process that is resource-intensive (in real-world applications, deriving an accurate first-principles model of a complex system is often challenging). In the era of big data, data-driven machine-learning approaches have emerged as viable alternatives to first-principles models within model-based optimization formulations. This advancement facilitates the practical application of model-based optimization across various industries, significantly enhancing its commercial viability. Neural networks, in particular, have been used to develop process models for complex systems where first-principles models are unavailable. 

\begin{figure}[ht!]
\centering
\includegraphics[width=0.7\textwidth]{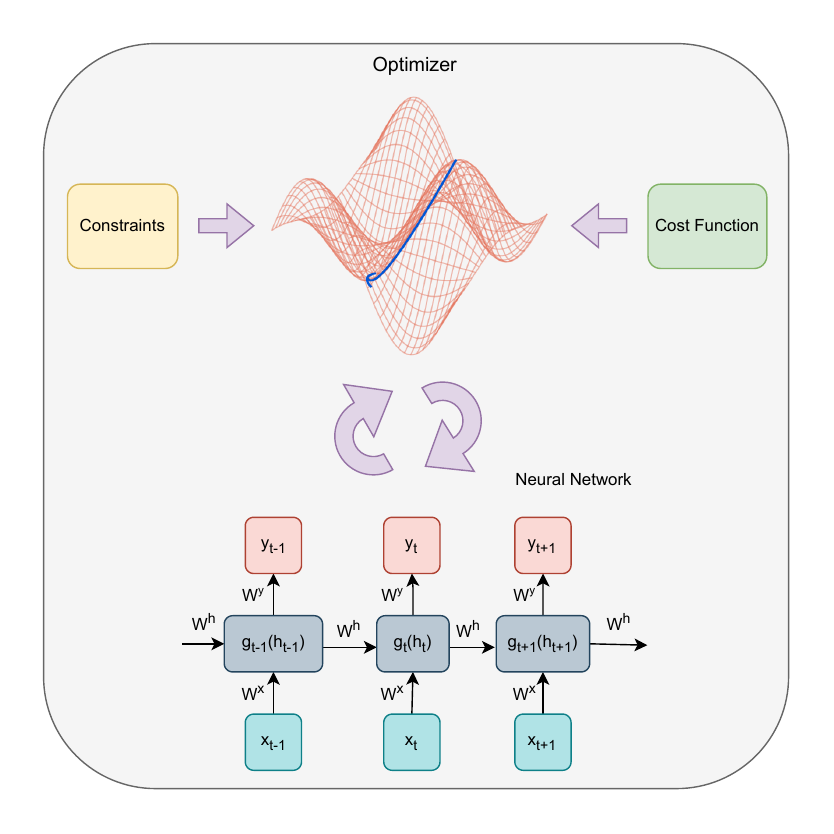}
\caption{System architecture of neural network-based optimization.}
\label{fig_optimizer}
\end{figure}

Neural network-based optimization (see Fig. \ref{fig_optimizer}) has found applications in various domains, such as approximation of the hybrid neuroprosthesis system \citep{bao2017recurrent}, regulation of Heating, Ventilation and Air-conditioning (HVAC) systems \citep{afram2017artificial,ellis2020encoder}, building energy optimization \citep{smarra2018data, yang2020model, bunning2022physics}, batch crystallization process \citep{zheng2022machine,zheng2022online}, and thin-film decomposition of quantum dot \citep{sitapure2022neural}. However, traditional neural network-based optimization and control encounter challenges in computational efficiency for online implementation. This is because using conventional neural networks to capture system dynamics within an optimization problem can introduce non-convexity. In our previous works \citep{wu2019machine1, wu2019machine2, pravin2022hyperparameter}, we noted that recurrent neural network (RNN)-based optimization for energy and chemical systems exhibited a significantly slower computational speed compared to optimization based on the first-principles model. 

While neural networks offer advantages in process modeling, ensuring computational efficiency (i.e., the runtime required to solve neural network-based optimization problems) is crucial for real-time optimization tasks, which can sometimes hinder their application in real-world systems. In hybrid energy systems, such as integrated solar photovoltaic (PV), battery, and grid systems, real-time or near-real-time control is essential to ensure efficient, reliable, and sustainable operation, and optimization techniques are often required. Additionally, in chemical industries, swift decision-making is pivotal for safety in chemical processes, as delays in addressing reactant changes can result in undesired reactions or unsafe conditions. Rapid decision-making extends its benefits to optimizing the utilization of raw materials, energy, and other resources in other industries as well, ultimately yielding cost savings and reducing the environmental footprint. In summary, for neural network-based optimization, runtime is a critical parameter to safeguard product quality, safety, efficiency, and resource utilization, with profound implications for both operational and environmental sustainability.

Inspired by the fact that the optima of convex optimization problems are easier and faster to obtain than those of non-convex optimization problems, our goal is to preserve the convexity in neural network-based optimization. This approach aims to ensure that the neural network output remains convex with respect to the input. Input Convex Neural Networks (ICNNs) were initially developed to ensure the achievement of globally optimal solutions by preserving system convexity, making them a powerful tool in the field of optimization and control. ICNNs have been applied to several neural network-based optimization problems, such as optimal transport mapping \citep{makkuva2020optimal}, voltage regulation \citep{chen2020data, chen2020input}, the Van de Vusse reactor \citep{yang2021optimization}, molecular discovery \citep{alvarez2021optimizing}, and DC optimal power flow \citep{zhang2021convex}. However, current versions of ICNNs (i.e., Input Convex Feedforward Neural Networks (ICFNN) \citep{amos2017input} and Input Convex Recurrent Neural Networks (ICRNN) \citep{chen2018optimal}) have not yet achieved the desired computational efficiency. For example, ICRNN performs comparably to conventional Long Short-Term Memory (LSTM) models in some optimization tasks due to LSTM's advanced gating architecture, which has been well documented in the literature \citep{shewalkar2019performance, sherstinsky2020fundamentals}.

Therefore, in this study, by combining the strengths of the LSTM architecture with the benefits of convex optimization, we propose a novel Input Convex LSTM (IC-LSTM) network to enhance the computational efficiency of neural network-based optimization. We validated the performance of the IC-LSTM-based optimization and control on a solar PV energy system at LHT Holdings in Singapore, and a chemical reactor example. The rest of this paper is organized as follows: Section \ref{sec1.5} introduces nonlinear systems and model-based optimization. Section \ref{sec2} provides a comprehensive overview of variants of RNNs and ICNNs, and proposes a novel IC-LSTM architecture, along with the underlying design principles. Section \ref{sec: iclstm} delves into the proof of preservation of convexity for IC-LSTM, provides an implementation guide for the IC-LSTM cell, and evaluates its modeling performance on surface fitting for non-convex bivariate scalar functions. Section \ref{sec3} proves the preservation of convexity in IC-LSTM-based optimization. Section \ref{sec_solarpv} and Section \ref{sec4} validate the performance and computational efficiency of our proposed framework against established baselines through case studies involving a solar PV energy system at LHT Holdings in Singapore and a continuous stirred tank reactor (CSTR), respectively.

\begin{remark}
In our previous works \citep{wu2019machine1, wu2019machine2}, we demonstrated that machine learning-based optimization can effectively replicate the results of first-principles model-based optimization, particularly in the context of using model predictive control (MPC) for CSTR. Since obtaining an accurate first-principles model or identifying the parameter values for such models is often challenging in real-world applications, neural networks provide an alternative solution to modeling complex nonlinear systems using a large, representative dataset. Additionally, when some physics are known, physics-informed machine learning could be applied to improve generalization performance (e.g., \citep{zheng2023physics}). In this work, we focus on black-box scenarios, and our goal is to accelerate neural network-based optimization during online implementation stages (e.g., real-time MPC optimization problems).
\end{remark}

\section{Nonlinear Systems and Optimization}
\label{sec1.5}
\subsection{Notation}
In the following sections, we adopt the common notation style in the deep learning community and use boldfaced symbols to denote vectors or matrices. $g$ denotes the activation function. The class $\mathcal{C}^1 $ denotes continuously differentiable functions. Set subtraction is denoted by ``$\backslash$'', that is, $A \backslash B := \{x | \ x \in A, x \notin B\}$. A matrix $\v M\in\R[n\times n]$ is positive semi-definite if $\vt v \v M \v v \geq 0, \forall \v v\in\R[n]$, and is denoted as $\v M\succeq 0$. Element-wise multiplication (i.e., Hadamard product) is denoted by $\ast$. Moreover, $\v f$ denotes the forget gate, $\v i$ denotes the input gate, $\v o$ denotes the output gate, $\v c$ denotes the cell state, and $\v h$ denotes the hidden state in the LSTM network.

\subsection{Class of Systems}\label{subsec:classofsystems}
In this work, we consider a general class of continuous-time nonlinear dynamical systems represented by the following  first-order ordinary differential equations (ODEs):
\begin{equation}\label{eq:classofsystems}
\vdot x(t) = F(\v x(t),\v u(t))
\end{equation}
where $\v x \in \R[n_x] $ denotes the state vector, $\v u \in \R[n_u]$ is the manipulated input. $F : D \times U  \to \R[n_x]$ is a $\mathcal{C}^1 $ function, where $D \subset \R[n_x]$ and $U \subset \R[n_u]$ are compact and connected subsets that contain an open neighborhood of the origin, respectively. For complex real-world systems, it is often tedious and difficult to derive the exact ODEs. Thus, when data are abundant, we can apply the data-driven approach to learn a neural network model, which can be utilized in an optimization problem.


\subsection{Neural Network-Based Optimization}
A machine learning-based MPC problem is formulated as follows:
\begin{subequations}\label{eq9}
\begin{gather}
\mathcal{L} = \min_{\v u \in S(\Delta)} \int_{t_k}^{t_{k+N}}J(\vtilde{x}(t),\v u(t))dt \label{eq9a}\\
\text{s.t. }\dot{\vtilde{x}}(t) = F_{nn}(\vtilde{x}(t),\v u(t))\label{eq9b}\\
\v u(t) \in U, \ \forall t \in [t_k,t_{k+N})\label{eq9c}\\
\vtilde{x}(t_k) = \v x(t_k)\label{eq9d}
\end{gather}
\end{subequations}
where $\vtilde{x}$ is the predicted state trajectory, $S(\Delta)$ is the set of piecewise constant functions with sampling period $\Delta$, and $N$ is the number of sampling periods in the prediction horizon. The objective function $\mathcal{L}$ in Eq. \eqref{eq9a} incorporates a cost function $J$ in terms of the system states $\v x$ and the control actions $\v u$. The dynamic function $F_{nn}(\vtilde{x}(t),\v u(t))$ in Eq. \eqref{eq9b} is parameterized as RNNs (e.g., plain RNN, ICRNN, IC-LSTM, etc., which will be introduced in the next section). Eq. \eqref{eq9c} is the constraint function $U$ on feasible control actions. Eq. \eqref{eq9d} defines the initial condition $\vtilde{x}(t_k)$ of Eq. \eqref{eq9b}, which is the state measurement at $t = t_k$. The first element of the optimal input trajectory computed by Eq. \eqref{eq9} will be applied to the system over the next sampling period, and the optimization problem will be resolved again at the next sampling time with new state measurement.

However, due to the inherent non-convexity of neural networks, neural network-based optimization problems are generally non-convex. Non-convex optimization is a challenging and time-consuming task, often necessitating a trade-off between solution accuracy and computational feasibility. This complexity arises from the presence of multiple local optima and intricate landscapes that are difficult to navigate. This challenge motivates us to develop input convex neural networks, aiming to transform the resulting non-convex neural network-based optimization problem into a convex one. By achieving convexity, we can solve the optimization process in a more tractable and computationally efficient way. 

\section{Family of Recurrent Neural Networks and Input Convex Neural Networks}
\label{sec2}
In this section, we provide a general introduction to conventional RNNs and their variants. Then, we provide a brief recap of the existing ICNNs in the literature. 

\subsection{Recurrent Neural Networks}
RNNs are a class of artificial neural networks designed for processing sequences of data. Unlike traditional feedforward neural networks, RNNs have connections that form directed cycles, allowing them to maintain a memory of previous inputs. This capability makes RNNs particularly well suited for tasks involving time series data, natural language processing, speech recognition, and other applications where the context provided by previous inputs is crucial for accurate predictions. RNNs leverage their internal state to capture temporal dynamics, enabling them to model complex sequential relationships effectively. Currently, two primary variants of RNNs are widely used in engineering fields: the simple RNN and the LSTM network, which are both non-convex in nature.

\subsubsection{Simple RNN}
A simple RNN cell follows:
\begin{align*}\label{eq_rnn}
\v h_t & = g_1(\v W^{(x)}\v x_t + \v U^{(h)}\v h_{t-1} + \v b^{(h)}) \\
y_t & = g_2(\v W^{(y)}\v h_t + \v b^{(y)})
\end{align*}
where $\v h_t$ is the hidden state at time step $t$, $\v x_t$ is the input at time step $t$, $\v h_{t-1}$ is the hidden state from the previous time step, $\v W^{(x)}$, $\v U^{(h)}$ and $\v W^{(y)}$ are weight matrices for the input, hidden state, and output respectively, $\v b^{(h)}$ and $\v b^{(y)}$ are the bias vectors for the hidden state and output respectively, and $y_t$ is the output at time step $t$.

\subsubsection{LSTM}
A conventional LSTM cell follows \citep{hochreiter1997long}:
\begin{subequations}\label{eq_lstm}
\begin{align}
\v f_t &= g^\ri{f}[\v U^\ri f\v x_t + \v W^\ri f\v h_{t-1} + \v b^\ri f]\\
\v i_t &= g^\ri{i}[\v U^\ri i\v x_t + \v W^\ri i\v h_{t-1} + \v b^\ri i]\\
\v o_t &= g^\ri{o}[\v U^\ri o\v x_t + \v W^\ri o\v h_{t-1} + \v b^\ri o]\\
\vtilde c_t &= g^\ri{c}[\v U^\ri c\v x_t + \v W^\ri c\v h_{t-1} + \v b^\ri c]\\
\v c_t &= \v f_t\ast\v c_{t-1} + \v i_t\ast\vtilde c_t\\
\v h_t &= \v o_t\ast g^\ri{h}(\v c_t) \\
y_t & = g^\ri{y}(\v W^{(y)}\v h_t + \v b^{(y)})
\end{align}
\end{subequations}
where $\v x_t$ is the input at time step $t$, $y_t$ is the output at time step $t$, $\v W^{(f)}$, $\v W^{(i)}$, $\v W^{(o)}$, $\v W^{(c)}$, $\v U^{(f)}$, $\v U^{(i)}$, $\v U^{(o)}$, $\v U^{(c)}$ and $\v W^{(y)}$ are weight matrices for different gates and outputs, respectively, $\v b^{(f)}$, $\v b^{(i)}$, $\v b^{(o)}$, $\v b^{(c)}$ and $\v b^{(y)}$ are the bias vectors for different gates and outputs, respectively.

\subsection{Input Convex Neural Networks}
ICNNs represent a category of deep learning models where the output is designed to exhibit convexity with respect to the input. Currently, there exist two primary variants of input convex architectures: Input Convex Feedforward Neural Networks and Input Convex Recurrent Neural Networks.

\subsubsection{ICFNN}
Standard feedforward neural networks (FNN) are generally non-convex due to the presence of multiple layers with nonlinear activation functions. Thus, ICFNN was proposed by \cite{amos2017input} with the output of each layer as follows:
\begin{equation*}\label{eq_icfnn}
\v z_{l+1} = g_l(\v W_l^\ri{z}\v z_l + \v W_l^\ri{x}\v x + \v b_l),\quad l=0,1,\dots,L-1
\end{equation*}
and with $\v z_0,\v W_0^\ri{z}=\v 0$. The output $\v z_{l+1}$ is input convex for single-step prediction if all weights $\v W_l^{(z)}$ are non-negative and all activation functions $g_l$ are convex and non-decreasing \citep{amos2017input}, while the output $\v z_{l+1}$ is input convex for multi-step ahead predictions if all weights $\v W_l^{(z)}$ and $\v W_l^{(x)}$ are non-negative and all activation functions $g_l$ are convex and non-decreasing \citep{bunning2021input}.

\subsubsection{ICRNN}
Following the idea of ICFNN, \cite{chen2018optimal} developed ICRNN with the output following the equations below:
\begin{align*}\label{eq_icrnn}
\v h_t & = g_1(\v U\vhat x_t + \v W\v h_{t-1} + \v D_2\vhat x_{t-1}) \\
y_t & = g_2(\v V\v h_t + \v D_1\v h_{t-1} + \v D_3\vhat x_t)
\end{align*}
where the output $y_t$ is input convex if all weights $\{\v U, \v W, \v V, \v D_1, \v D_2, \v D_3\}$ are non-negative and all activation functions $g_i$ are convex and non-decreasing, where $\vhat x_t$ denotes the expanded input $\bmat{\vt x_t, -\vt x_t}^\T$.

\section{Input Convex Long Short-Term Memory}
\label{sec: iclstm}
In this section, we first propose a novel input convex network in the context of LSTM, and then theoretically prove the convex property of the IC-LSTM. It is important to note that 
designing IC-LSTM is more complex than ICRNNs or ICNNs, as non-negative weight constraints and convex, non-decreasing activation functions do not guarantee model convexity, which will be elaborated in the following subsections. Subsequently, a brief coding implementation guide for the IC-LSTM cell in Python is provided, with several toy examples for surface fitting of non-convex bivariate scalar functions to demonstrate the modeling performance using the proposed IC-LSTM.

\subsection{Input Convex LSTM}
LSTM networks have been shown to provide many advantages over traditional RNNs and FNNs, particularly in handling sequential data and capturing long-term dependencies \citep{hochreiter1997long, shewalkar2019performance, sherstinsky2020fundamentals}.  By controlling the flow of information and retaining relevant context over time, LSTMs are well suited for various tasks involving temporal dynamics and sequential patterns. Therefore, inspired by the success of ICNNs and ICRNNs, in this work, we develop a novel input convex architecture based on LSTM, referred to as Input Convex LSTM, as shown in Fig. \ref{fig_IC-LSTM}. Specifically, the output of the IC-LSTM cell follows (see Fig. \ref{fig_IC-LSTM_cell}):
\begin{subequations}\label{eq_IC-LSTM_cell}
\begin{align}
\v f_t &= g^\ri{g}[\v D^\ri f(\v W^\ri x\vhat x_t + \v W^\ri h\v h_{t-1}) + \v b^\ri f]\\
\v i_t &= g^\ri{g}[\v D^\ri i(\v W^\ri x\vhat x_t + \v W^\ri h\v h_{t-1}) + \v b^\ri i]\\
\v o_t &= g^\ri{g}[\v D^\ri o(\v W^\ri x\vhat x_t + \v W^\ri h\v h_{t-1}) + \v b^\ri o]\\
\vtilde c_t &= g^\ri{c}[\v D^\ri c(\v W^\ri x\vhat x_t + \v W^\ri h\v h_{t-1}) + \v b^\ri c]\\
\v c_t &= \v f_t\ast\v c_{t-1} + \v i_t\ast\vtilde c_t\\
\v h_t &= \v o_t\ast g^\ri{c}(\v c_t)
\end{align}
\end{subequations}
where $\v D^\ri f, \v D^\ri i, \v D^\ri o, \v D^\ri c\in\R[n_h\times n_h]$ are diagonal matrices with non-negative entries. $\v W^\ri h\in\R[n_h\times n_h]$ and $\v W^\ri x\in\R[n_h\times n_i]$ are non-negative weights (i.e., sharing weights across all gates), and $\v b^\ri f, \v b^\ri i, \v b^\ri o, \v b^\ri c\in\R[n_h]$ are the bias. Similar to \cite{chen2018optimal}, we \textbf{expand the input} as $\vhat x_t = \bmat{\vt x_t, -\vt x_t}^\T\in\R[n_i]$, where $n_i=2n_x$.

Furthermore, the output of $L$-layer IC-LSTM follows (see Fig. \ref{fig_IC-LSTM_nlayer}):
\begin{align}\label{eq_IC-LSTM}
\v z_{t,l} &= g^{(d)}[\v W_l^\ri{d}\v h_{t, l} + \v b_l^\ri{d}] + \vhat x_t,\quad l=1,2,\dots,L\\
\v y_t &= g^{(y)}[\v W^\ri{y}\v z_{t,L} + \v b^\ri{y}]
\end{align}
where $\v W_l^\ri d\in\R[n_i\times n_h]$ and $\v W^\ri y\in\R[n_o\times n_i]$ are the non-negative weights; $\v b_l^\ri d\in\R[n_i]$ and $\v b^\ri y\in\R[n_o]$ are the bias. $g^{(d)}$ is any convex, non-negative, and non-decreasing activation function, and $g^{(y)}$ is convex and non-decreasing.

\begin{figure}[ht!]
    \centering
    \begin{subfigure}[t]{0.76\textwidth}
        \centering
        \includegraphics[width=\columnwidth]{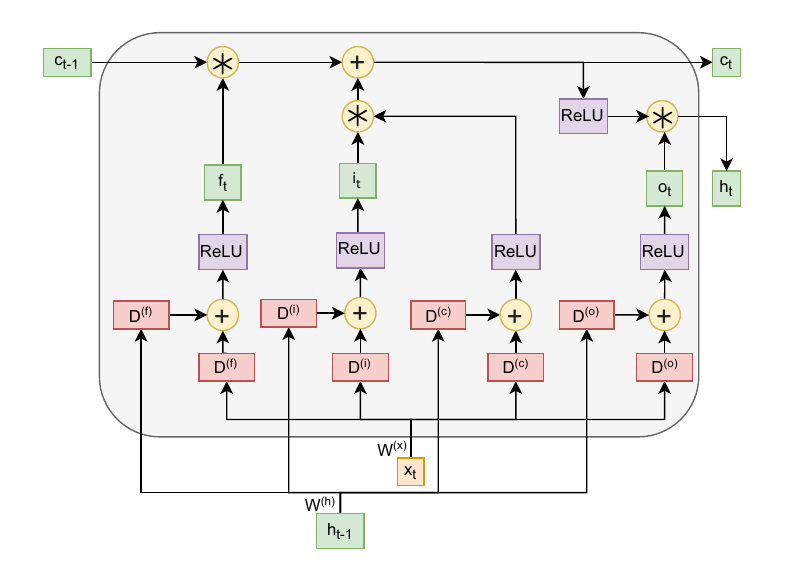}
        \caption{IC-LSTM cell}
        \label{fig_IC-LSTM_cell}
    \end{subfigure}
    ~ 
    \begin{subfigure}[t]{0.21\textwidth}
        \centering
        \includegraphics[width=\columnwidth]{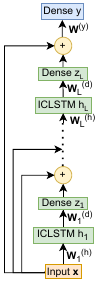}
        \caption{L-layer IC-LSTM}
        \label{fig_IC-LSTM_nlayer}
    \end{subfigure}
    \caption{Architecture of IC-LSTM.}
    \label{fig_IC-LSTM}
\end{figure}

As discussed in \cite{chen2018optimal}, expanding the input to $\vhat x$ facilitates network composition in dynamic system scenarios, and provides additional advantages. In our experiments, we discovered that incorporating a non-negative weight constraint in an ICNN restricts its representability. The expanded input allows for a more accurate representation of dynamic systems, as opposed to the original input. This input expansion can be regarded as a form of data augmentation, bolstering the model's robustness. Furthermore, including the negation of the input enhances gradient flow during training. By providing the network with both $\v x$ and $-\v x$, we introduce a larger number of symmetric data points, resulting in more consistent and well balanced gradients throughout the training process.

Additionally, it should be pointed out that the design of IC-LSTM is not as trivial as ICRNNs or ICNNs. Imposing non-negative constraints on LSTM weights and requiring activation functions to be convex and non-decreasing do not guarantee the convexity of LSTM models. Specifically,  unlike traditional LSTM models, the proposed IC-LSTM employs \textbf{shared weights} across all gates and introduces non-negative \textbf{trainable scaling vectors} to distinguish them. Moreover, the IC-LSTM enforces an additional non-negative constraint on the activation function compared to ICFNN and ICRNN. These modifications ensure the input convexity of the IC-LSTM cell. Furthermore, ICFNN and ICRNN leverage weighted direct ``passthrough'' layers to improve representational capabilities, while IC-LSTM adopts a parameter-free skip connection to enhance generalization \citep{he2016deep}. In particular, a dense layer with the same dimension as the input is followed by every LSTM layer to maintain consistent dimensions between the input and output of the LSTM layer. This configuration facilitates the subsequent concatenation of the layer output with the input via the skip connection. It should be noted that the use of parameter-free skip connections and dense layers reduces network complexity. This simplification is particularly beneficial as it compensates for the internal complexity inherent to the LSTM layer. In the following section, we will prove the convexity of the proposed IC-LSTM network.

\begin{remark}   
Due to the use of the ReLU activation function in ICNNs, weight initialization is crucial for effective learning and generalization. Poor initialization can result in suboptimal modeling performance and potentially lead to exploding gradients. A simple but effective practice is to initialize the weights small and close to zero (e.g., using a random normal distribution with a mean of 0 and a standard deviation of 0.01, or a random uniform distribution with a minimum of 0 and a maximum of 1). For a more comprehensive study, interested readers can refer to \cite{hoedt2024principled} on weight initialization for ICNNs.
\end{remark}

\subsection{Convexity of IC-LSTM}
In this subsection, we prove the convexity of IC-LSTM. The following lemma is first provided and will be used in the proof of Theorem \ref{theorem1} on the convexity of an $L$-layer IC-LSTM.
\begin{lemma}\label{lemma: IC-LSTM cell}
The proposed IC-LSTM cell depicted in Fig. \ref{fig_IC-LSTM_cell} is convex and non-decreasing from inputs to outputs (hidden states), if all the weights, i.e., $\v W^\ri h, \v W^\ri x, \v D^\ri f, \v D^\ri i,\v D^\ri o$ and $\v D^\ri c$, are non-negative and all the activation functions, i.e., $g^\ri g, g^\ri c$, are smooth, convex, non-decreasing and non-negative.
\end{lemma}

\begin{proof}
The non-decreasing property is trivial since all weights are non-negative and all activation functions are non-decreasing. To prove the convexity, we compute the second derivatives and check for the positive semi-definiteness. To make the derivation more notationally clear, we neglect the time index subscript and use $(\diamond)^{\tau -},\ \tau=1,2,3,\dots$ to denote the quantity $\diamond$ in the previous $\tau^{th}$ time step. The number is dropped when $\tau=1$. By omitting the bias terms, the conventional LSTM cell is written as:
\begin{subequations}
\begin{align}
\v i &= g_i(\v A\v x + \v B\v h^-)\\
\v f &= g_f(\v D\v x + \v E\v h^-) \\
\v o &= g_o(\v M\v x + \v U\v h^-) \\
\vtilde c &= g_{\tilde c}(\v V\v x + \v W\v h^-)\\
\v c &= \v f\ast\v c^- + \v i\ast\vtilde c\\
\v h &= \v o\ast\vbar c,\  \vbar c = g_c(\v c)
\end{align}
\end{subequations}

The statement that $\v h$ is convex with respect to $\v x$ implies that each component $h_j(\v x)$ is convex with respect to $\v x$. Without explicitly referring to which $j^{th}$ component, the Hessian matrix\footnote{Except for the gradient $\nabla f$ which is a column vector, we adopt the row-major or numerator layout convention in matrix calculus.} is expressed as:
\begin{equation}\label{eq: hess}
\nabla^2_{\v x} h = \nabla^2 h =  o\nabla^2\bar c + \bar c\nabla^2o + \nabla o(\nabla\bar c)^\T + \nabla\bar c(\nabla o)^\T
\end{equation}
Substituting Eq. \eqref{eq: hess} with the following expressions, where prime denotes the derivative and $\v a,\v d, \v m, \v v$ are transposed row vectors of corresponding matrices,
\begin{align*}
\nabla o &=  g'_o\v m\\
\nabla \bar c &=  g'_c\nabla c\\
\nabla c &= c^-\nabla f + f\nabla c^- + \tilde c\nabla i+ i\nabla\tilde c\\
&= c^- g'_f\v d + \tilde c g'_i\v a + i g'_{\tilde c}\v v\\
\nabla^2 o &=  g''_o\v m\vt m\\
\nabla^2\bar c &=  g''_c\nabla c(\nabla c)^\T +  g'_c\nabla^2 c\\
\nabla^2 c &= c^- g''_f\v d\vt d + \tilde c g''_i\v a\vt a + i g''_{\tilde c}\v v\vt v +  g'_{\tilde c} g'_i(\v a\vt v + \v v\vt a)
\end{align*}
we obtain:
\begin{align}\label{eq: hess_x h}
\nabla^2_{\v x} h
=& [o g''_c(c^- g'_f)^2 + oc^- g'_c g''_f]\v d\vt d + [o g''_c(\tilde c g'_i)^2 + o\tilde c g'_o g''_i]\v a\vt a+ [o g''_c(i g'_{\tilde c})^2 + oi g'_c g''_c]\v v\vt v + \bar c g''_o\v m\vt m\nonumber\\
&+ o\tilde c c^- g''_c g'_i g'_f(\v a\vt d + \v d\vt a) + oic^- g''_c g'_{\tilde c} g'_f(\v d\vt v + \v v\vt d) +  g'_{\tilde c} g'_i(oi\tilde c g''_c + o g'_c)(\v a\vt v + \v v\vt a)\nonumber\\
&+  g'_o g'_c[c^- g'_f(\v m\vt d + \v d\vt m) + \tilde c g'_i(\v m\vt a + \v a\vt m) + i g'_{\tilde c}(\v m\vt v + \v v\vt m)]
\end{align}

In the proposed IC-LSTM cell, note that all activation functions $g$ are convex, non-decreasing, and non-negative, and $\v a = \alpha_x\v d = \beta_x \v m = \gamma_x\v v,\ \alpha_x,\beta_x,\gamma_x\geq 0$ due to the weight sharing and non-negative scaling. Thus, $\nabla^2_{\v x}h$ reduces to $\lambda \v a\vt a$ for some $\lambda \geq 0$, which is positive semi-definite.

Furthermore, to render IC-LSTM convex to all its inputs, we require $h$ to be convex with respect to the previous input $\v x^-$ (and any past input, which will be discussed later). We check the convexity via the Hessian matrix of $h$ with respect to $\v x^-$,
\begin{equation}\label{eq: hess_x- h}
\nabla^2_{\v x^-} h = (\nabla_{\v x^-}\v h^-)^\T\nabla^2_{\v h^-}h(\nabla_{\v x^-}\v h^-) + \sum\limits_{j=1}^{n_h}(\partial_{h^-_j}h)\nabla^2_{\v x^-}h^-_j
\end{equation}

Note that $\nabla^2_{\v x^-}h^-_j$ is a shift of time index of Eq. \eqref{eq: hess_x h}, and thus $\nabla^2_{\v x^-}h^-_j\succeq 0$ is satisfied for the IC-LSTM cell as discussed earlier. Additionally, we obtain $\nabla^2_{\v h^-}h$ by mirroring $\nabla^2_{\v x}h$ as follows:
\begin{align}\label{eq: hess_h- h}
\nabla^2_{\v h^-} h
=& [o g''_c(c^- g'_f)^2 + oc^- g'_c g''_f]\v e\vt e + [o g''_c(\tilde c g'_i)^2 + o\tilde c g'_o g''_i]\v b\vt b+ [o g''_c(i g'_{\tilde c})^2 + oi g'_c g''_c]\v w\vt w + \bar c g''_o\v u\vt u\nonumber\\
&+ o\tilde c c^- g''_c g'_i g'_f(\v b\vt e + \v e\vt b) + oic^- g''_c g'_{\tilde c} g'_f(\v e\vt w + \v w\vt e) +  g'_{\tilde c} g'_i(oi\tilde c g''_c + o g'_c)(\v b\vt w + \v w\vt b)\nonumber\\
&+  g'_o g'_c[c^- g'_f(\v u\vt e + \v e\vt u) + \tilde c g'_i(\v u\vt b + \v b\vt u) + i g'_{\tilde c}(\v u\vt w + \v w\vt u)]
\end{align}
Again, due to the weight sharing and non-negative scaling in the IC-LSTM cell, we have $\v b=\alpha_h\v e=\beta_h\v u=\gamma_h\v w,\ \alpha_h,\beta_h,\gamma_h\geq 0$. Therefore, $\nabla^2_{\v h^-}h = \lambda \v b\vt b\succeq 0$ for some $\lambda\geq 0$.

Lastly, we check for the gradient:
\begin{align}\label{eq: grad_h- h}
\nabla_{\v h^-}h &= \bar c\nabla_{\v h^-} o + o \nabla_{\v h^-}\bar c\nonumber\\
&=\bar c g'_o\v u + o g'_c(c^- g'_f\v e + \tilde c g'_i\v b + i g'_c\v w)
\end{align}
Since all activation functions $g$ are non-decreasing and non-negative, and $u_j,e_j,b_j,w_j\geq 0,\ \forall j$ in the IC-LSTM cell, we have $\partial_{h_j^-}h \geq 0$. Combining all results, $\nabla^2_{\v h^-}h\succeq 0$, $\nabla^2_{\v x^-}h^-_j\succeq 0$, and $\partial_{h^-_j}h\geq 0$, we arrive at $\nabla^2_{\v x^-} h\succeq 0$.

Similarly, it can be found that the derivations of $\nabla^2_{\v x^{\tau-}} h$ for $\tau = 2,3,\dots$ (i.e., Hessians with respect to past inputs) reveal the same patterns as in Eq. \eqref{eq: hess_x- h} due to the recurrent structure of the model. For example, $h$ is convex with respect to the input $\v x^{2-}$ two time steps in the past, i.e., $\nabla^2_{\v x^{2-}}h \succeq 0$ when $\nabla^2_{\v h^-}h\succeq 0, \nabla^2_{\v h^{2-}}h^-_j\succeq 0, \nabla^2_{\v x^{2-}}h^{2-}_k\succeq 0$ and $\partial_{h^-_j}h,\partial_{h^{2-}_k}h^-_j\geq 0,\ \forall j,k$. We realize that all those conditions are satisfied since they are essentially the same as in Eq. \eqref{eq: hess_x h}, Eq. \eqref{eq: hess_h- h}, and Eq. \eqref{eq: grad_h- h} with a change in time index. Therefore, the IC-LSTM cell is input convex when the conditions in Lemma \ref{lemma: IC-LSTM cell} are satisfied.
\end{proof}

\begin{remark}
Without loss of generality, we assume that the activation functions are smooth. In practice, we can still use the rectified linear function, i.e., ReLU, since it is convex, non-decreasing, and non-negative, and is only non-smooth at the origin. Alternatively, we can choose the softplus, i.e., $\log(1+\exp(\beta x))/\beta,\ \beta > 0$, as a smooth approximation of the ReLU.
\end{remark}

Next, we develop the following theorem to show the convexity of $L$-layer IC-LSTM.
\begin{theorem}\label{theorem1}
Consider the $L$-layer IC-LSTM  as shown in Fig. \ref{fig_IC-LSTM}. Each element of the output $\v y_t$ is a convex, non-decreasing function of the input $\vhat x_\tau=\bmat{\vt x_\tau,-\vt x_\tau}^\T$ (or just $\v x_\tau$) at the time step $\tau=t,t-1,\dots,1$, for all $\vhat x_\tau \in D\times D$ in a convex feasible space if all of the following conditions are met: (1) All weights are non-negative; (2) All activation functions are convex, non-decreasing, and non-negative (e.g., ReLU), except for the activation function of the output layer which is convex and non-decreasing (e.g., ReLU, Linear, LogSoftmax).
\end{theorem}

\begin{proof}
With Lemma \ref{lemma: IC-LSTM cell}, the proof directly follows from the fact that affine transformations with non-negative matrices and compositions of convex non-decreasing functions preserve convexity \citep{boyd2004convex}. 
\end{proof}

\subsection{Implementation of IC-LSTM Cell}
In this subsection, we provide a brief TensorFlow Keras \citep{chollet2015keras} implementation of the IC-LSTM cell using Python, with the complete code available at \url{https://github.com/killingbear999/IC-LSTM}. 
Due to the architectural differences (i.e., the hidden state computation of the IC-LSTM cell (see Eq. \eqref{eq_IC-LSTM_cell}) differs from the LSTM cell (see Eq. \eqref{eq_lstm}) to preserve convexity), we customize the IC-LSTM cell using Keras's custom RNN layer. To customize the IC-LSTM cell, we first define all the weights and biases, treating scaling vectors as trainable weights. Specifically, we set the initialization techniques and constraints on all weights and biases. All weights (i.e., $\v W^\ri h, \v W^\ri x$) and scaling vectors (i.e., $\v D^\ri f, \v D^\ri i, \v D^\ri o, \v D^\ri c$) are trainable (i.e., can be updated during backpropagation) and subject to a non-negative constraint. It is recommended to initialize $\v W^\ri x$ using random normal initializer with a mean of 0 and a standard deviation of 0.01, $\v W^\ri h$ using orthogonal initializer or identity initializer with a gain of 0.1, $\v D^\ri f, \v D^\ri i, \v D^\ri o, \v D^\ri c$ using random uniform initializer with a minimum of 0 and a maximum of 1. All biases (i.e., $\v b^\ri f, \v b^\ri i, \v b^\ri o, \v b^\ri c$) are trainable without any constraints, and it is recommended to initialize them using a zero initializer. All initializers are available in Keras. Since the IC-LSTM requires non-negative constraints on weights, this constraint is enforced by clipping all negative values to 0 (see Listing \ref{listing1}).

Subsequently, we compute the hidden state according to Eq. \eqref{eq_IC-LSTM_cell}. Finally, we recursively compute the hidden state and the output, and update the weights using an optimizer such as Adam \citep{kingma2014adam}, which will be managed by Keras.

\begin{lstlisting}[language=Python, caption=Keras implementation of a non-negative constraint, label={listing1}]
import keras
from keras import ops

class NonNegative(keras.constraints.Constraint):
    def __call__(self, w):
        return w * ops.cast(ops.greater_equal(w, 0.), dtype=w.dtype)
\end{lstlisting}

\subsection{Toy Examples: Surface Fitting}
To demonstrate the input convexity of IC-LSTM, we trained the model for a 2-dimensional regression task. Three toy datasets were constructed based on the following non-convex bivariate scalar functions:
\begin{subequations}
\begin{align}
    \label{eq_f1} f_1(x,y) & = -\cos(4x^2 + 4y^2) \\
    \label{eq_f2} f_2(x,y) & = \max(\min(x^2 + y^2, (2x-1)^2+(2y-1)^2-2), -(2x+1)^2 - (2y+1)^2 +4) \\
    \label{eq_f3} f_3(x,y) & = x^2(4-2.1x^2+x^{\frac{4}{3}}) - 4y^2(1-y^2) + xy
\end{align}
\end{subequations}

Given its input convex architecture, IC-LSTM is expected to transform these non-convex functions into convex representations. As shown in Fig. \ref{fig_f1}, IC-LSTM exhibits input convexity in modeling functions $f_1$, $f_2$, and $f_3$, while it struggles to explicitly fit these functions.

\begin{figure}[ht!]
    \centering
    \begin{subfigure}[t]{0.31\textwidth}
        \centering
        \includegraphics[width=\columnwidth]{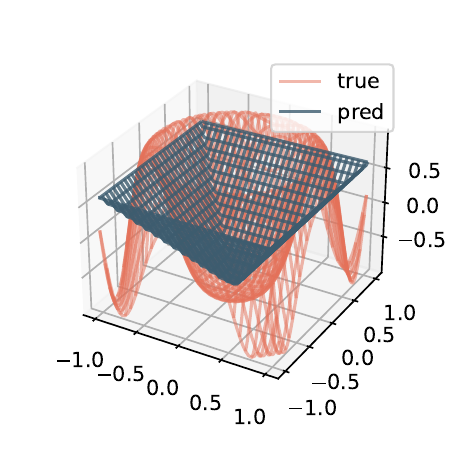}
        \caption{$f_1(x,y) = -\cos(4x^2 + 4y^2)$}
        \label{fig_f1_IC-LSTM}
    \end{subfigure}
    ~ 
    \begin{subfigure}[t]{0.31\textwidth}
        \centering
        \includegraphics[width=\columnwidth]{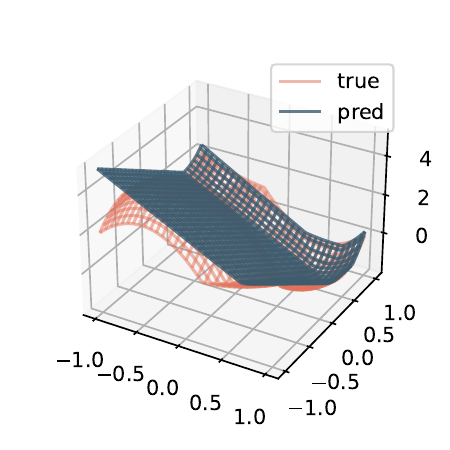}
        \caption{$f_2(x,y) = \max(\min(x^2 + y^2, (2x-1)^2+(2y-1)^2-2), -(2x+1)^2 - (2y+1)^2 +4)$}
        \label{fig_f2_IC-LSTM}
    \end{subfigure}
    ~
    \begin{subfigure}[t]{0.31\textwidth}
        \centering
        \includegraphics[width=\columnwidth]{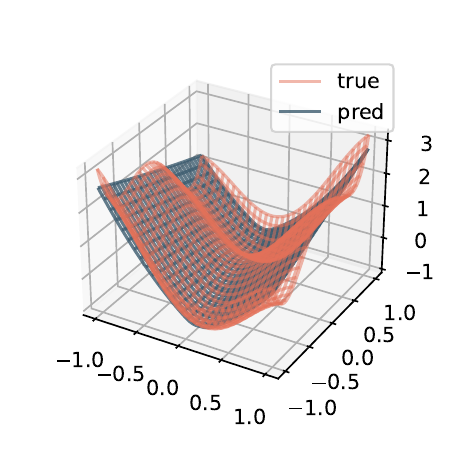}
        \caption{$f_3(x,y) = x^2(4-2.1x^2+x^{\frac{4}{3}}) - 4y^2(1-y^2) + xy$}
        \label{fig_f3_IC-LSTM}
    \end{subfigure}
    \caption{3D plots of bivariate scalar functions, where `true' represents the underlying non-convex function and `pred' represents the convex form learned by IC-LSTM.}
    \label{fig_f1}
\end{figure}

\begin{remark}
\label{remark4}
ICNNs offer benefits like global optimality and stability in optimization problems but may sacrifice accuracy in modeling highly non-convex functions due to their inherent convex nature. However, they remain effective for practical systems with relatively low non-convexity, providing efficient solutions for neural network-based optimization without compromising desired accuracy to a significant extent. Comparing ICNNs' testing losses with traditional neural networks helps assess their performance (i.e., achieving similar accuracy levels validates ICNNs as viable approximations for nonlinear systems). Additionally, leveraging partially input convex neural networks, as proposed by \cite{amos2017input}, can enhance the representative power of ICNNs by ensuring that the output remains convex with respect to specific input elements only. In conclusion, it is advisable for users to carefully evaluate the task requirements on a case-by-case basis when considering the use of ICNNs. Generally, ICNNs are recommended for real-time optimization tasks where computational speed is critical.
\end{remark}

\section{IC-LSTM-Based Optimization}
\label{sec3}
Under certain assumptions on the cost function and state constraints, ICNN models result in convex optimization problems that can be solved to a global optimum in real time \citep{bunning2022physics}. Thus, we utilize an IC-LSTM to model the state transition dynamics, as expressed by $\dot{\vtilde x}(t) = F_{nn}(\vtilde x(t),\v u(t))$ in Eq. \eqref{eq9b}. This neural network is then embedded into a finite-horizon optimization problem as designed in Eq. \eqref{eq9}. The primary objective of this integration is to determine the optimal sequence of actions, denoted as $\v u^*_t, \v u^*_{t+1}, \ldots, \v u^*_{t+N-1}$, for a predetermined prediction horizon $N$. We first present a lemma to show the sufficient conditions for the optimization problem in Eq. \eqref{eq9} to be convex.

\begin{lemma}\label{lemma6}
    By embedding IC-LSTM into Eq. \eqref{eq9}, the optimization problem is considered as a convex optimization problem if both the objective function and the constraints are convex.
\end{lemma}

Next, we develop the following theorem to show that a convex optimization problem with multi-step ahead prediction remains convex. 
\begin{theorem}\label{theorem2}
Consider the neural network-based convex MPC problem in Eq. \eqref{eq9}. The problem remains input convex in the face of multi-step ahead prediction (i.e., when the prediction horizon $N > 1$) if and only if the neural network embedded is inherently input convex and non-decreasing (e.g., IC-LSTM) and the neural network output is non-negative for certain objective functions (e.g., quadratic functions and absolute functions).
\end{theorem}

\begin{proof}
The proof of Theorem \ref{theorem2} is intuitive. Consider a $2$-step ahead prediction problem (i.e., $N = 2$) with a $L$-layer embedded IC-LSTM $\v f_t(\v x_t,\v u_t)$, the final output is $\v y_2 = \v f_2(\v x_2 = \v f_1(\v x_1, \v u_1), \v u_2)$, where $\v x$ is the input. It is equivalent to a $1$-step ahead prediction problem with a $2L$-layer embedded IC-LSTM but with a new input $\v u_2$ concatenated at the output of the $L^{th}$ layer. Without loss of generality from Theorem \ref{theorem1}, the $2$-step ahead prediction remains input convex. Hence, without loss of generality, a $N$-step ahead prediction problem with a $L$-layer embedded IC-LSTM is equivalent to a $1$-step ahead prediction problem with a $NL$-layer embedded IC-LSTM with new inputs $\v u_t$ concatenated at the output of every $L^{th}$ layer, which is indeed input convex.
\end{proof}

\begin{remark}
The output of IC-LSTM should be carefully designed to meet the specific requirements of the task. For example, to make the quadratic objective function $\mathcal{L}(\v u) = \v\phi(\v u)^\T\v Q\v\phi(\v u)$, where $\v \phi(\v u)$ is a vector-valued function, convex in terms of the vector $\v u$, we require $\v Q$ to be a positive semidefinite matrix of an appropriate dimension, each component of $\v\phi(\v u)$ to be a convex function with respect to the input vector, and all values $\phi_i(\v u)\geq 0$. To ensure non-negative outputs and maintain convexity in the quadratic objective, ReLU activation can be used in the output layer. Alternatively, training data can be preprocessed so that the IC-LSTM is trained to learn the absolute value of the system state \citep{wang2024fast}. In general, it should be noted that while applying ICNNs can contribute to convexity, it does not guarantee that the overall optimization task will be convex. Other factors, such as objective functions and constraints, also affect its convexity.
\end{remark}



\section{Application to a Solar PV Energy System}
\label{sec_solarpv}
\subsection{System Description}
In this case study, we apply the MPC control to a hybrid energy system in the context of LHT Holdings, a wood pallet manufacturing industry based in Singapore (see Fig. \ref{fig_pipeline} for a detailed manufacturing pipeline of LHT Holdings), with the aim of maximizing the supply of solar energy for environmental sustainability. Before the installation of the solar PV system, LHT Holdings relied solely on the main utility grid to fulfill its energy needs. The solar PV system was successfully installed by 10 Degree Solar in late 2022. Moreover, the Solar Energy Research Institute of Singapore (SERIS) and the Singapore Institute of Manufacturing Technology (SIMTech) have installed various sensors for monitoring the solar PV system. Data such as average global solar irradiance, ambient humidity, module temperature, wind speed, and wind direction are uploaded to an online system on a minute-by-minute basis (see Fig. \ref{fig_lht_solar} for the actual solar PV system and sensors). As illustrated in Fig. \ref{fig_lht}, the factory draws power from the solar PV system, the main power grid, and the battery. Specifically, the solar PV system serves as the primary energy source for the industrial facility, while the main utility grid and the batteries act as secondary energy sources to supplement any deficiencies in solar energy production. Any surplus solar energy generated beyond the current requirements is stored in batteries for future use.

\begin{figure}[ht!]
\centering
\includegraphics[width=\textwidth]{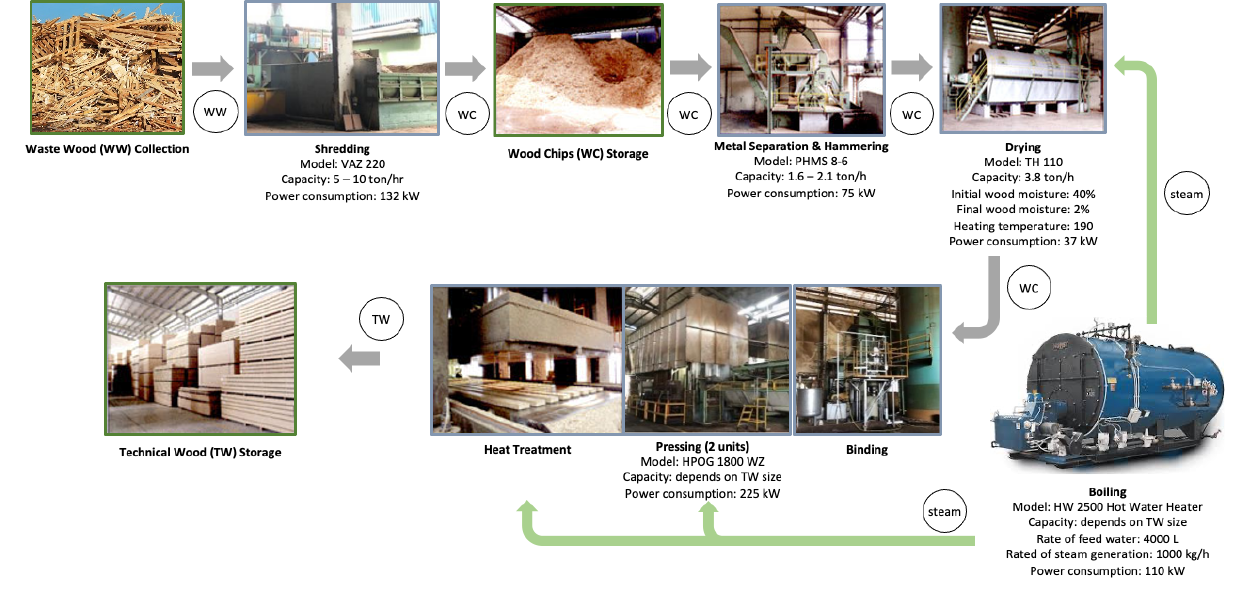}
\caption{LHT Holdings technical wood production pipeline.}
\label{fig_pipeline}
\end{figure}

\begin{figure}[ht!]
\centering
\includegraphics[width=0.7\textwidth]{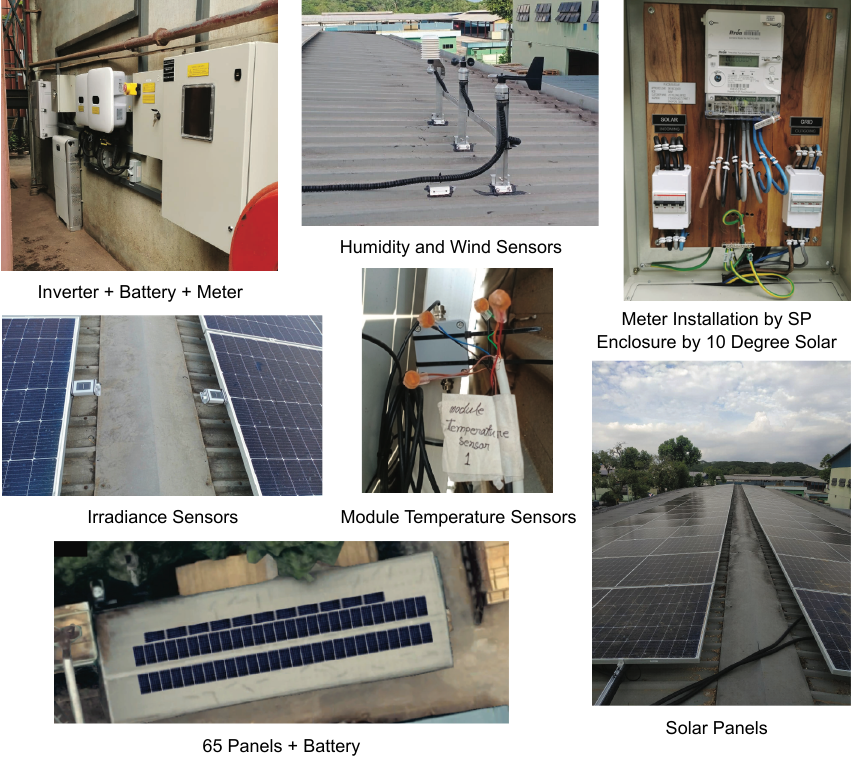}
\caption{LHT Holdings solar PV system.}
\label{fig_lht_solar}
\end{figure}

For the solar PV system, we adopt the \textit{solar PV-converter-battery} model (see Fig. \ref{fig: spvb_drawing}) from \cite{valenciaga2001power} and study the real-time control problem based on this model. The system consists of a solar PV panel, a buck DC/DC converter, and a battery connected in parallel to the panel.
\begin{figure}[ht!]
    \centering
    \begin{subfigure}[t]{0.47\textwidth}
        \centering
        \includegraphics[width=\columnwidth]{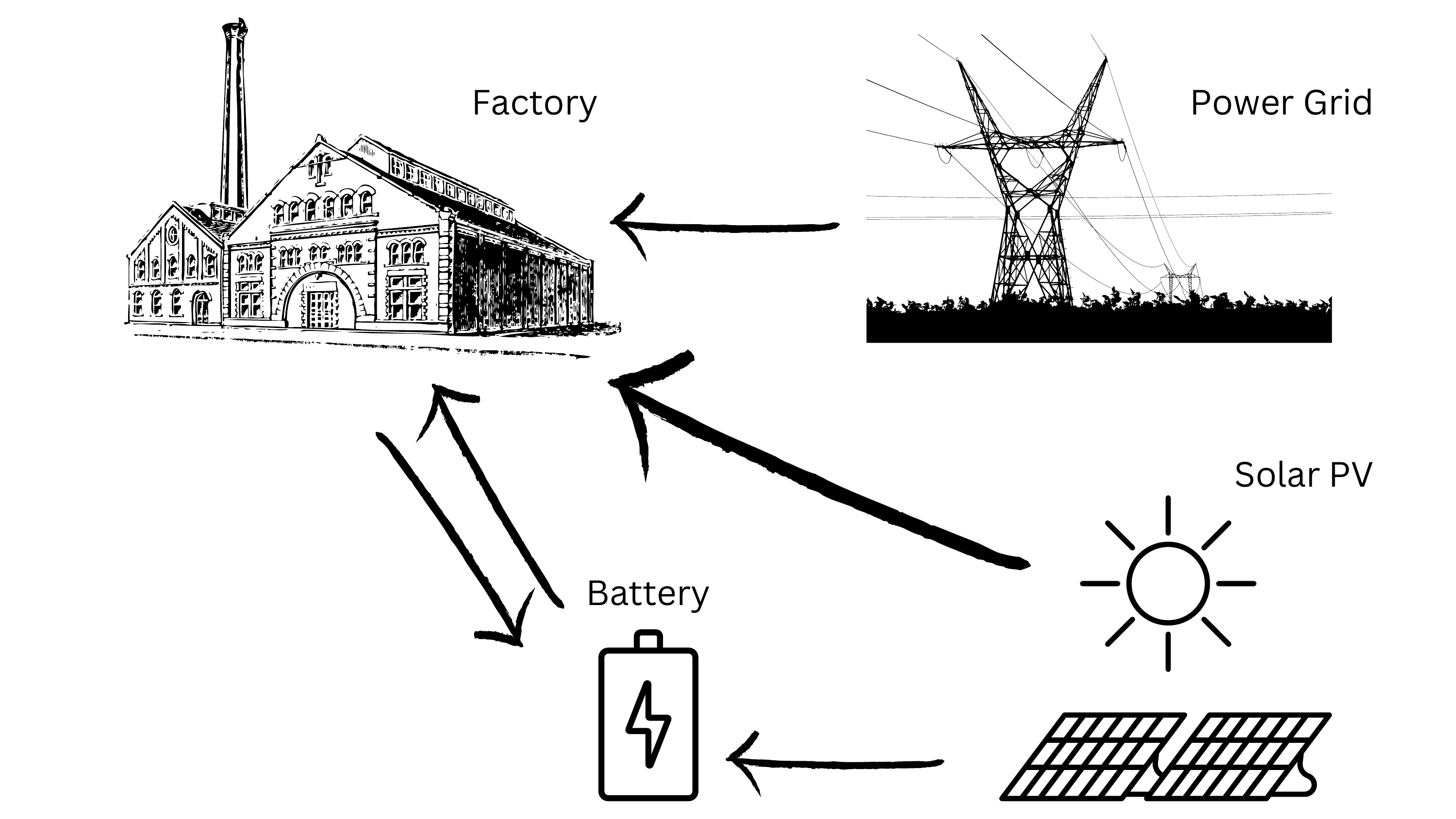}
        \caption{Energy system at LHT Holdings}
        \label{fig_lht}
    \end{subfigure}
    ~ 
    \begin{subfigure}[t]{0.48\textwidth}
        \centering
        \includegraphics[width=\columnwidth]{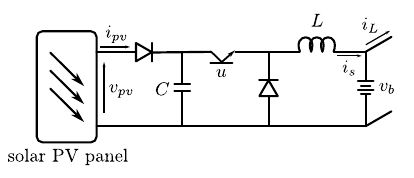}
        \caption{Illustration of the solar PV energy system}
        \label{fig: spvb_drawing}
    \end{subfigure}
    \caption{Schematics of (a) integrated solar PV, battery, factory, and power grid system at LHT Holdings, and (b) solar PV panel.}
    \label{fig_energy_system}
\end{figure}
The system's dynamics are governed by the following ODEs:
\begin{subequations}
\begin{align}
\dv{v_{pv}}{t} &= \frac{1}{C}(i_{pv} - i_s u)\\
\dv{i_s}{t}&=\frac{1}{L}(-v_b + v_{pv} u)\\
\dv{v_c}{t}&= \frac{1}{C_b}(i_s - i_L)
\end{align}
\end{subequations}
where $v_{pv}$ and $i_{pv}$ are the terminal voltage and the output current of the solar PV panel respectively; $i_s$ is the effective output current of the solar PV system; $G$ is the global horizontal irradiance; $T$ is the cell temperature; $v_c$ is the voltage across the internal capacitor of the battery; $i_L$ is the load current (i.e., factory's demand). From \cite{kumar2018solar}, the solar PV output current $i_{pv}$ is a function of $G,T$ and $v_{pv}$,
\begin{equation}
i_{pv} = n_p I_{ph} - n_p I_{s}\bqty{\exp\qty{\frac{q(v_{pv} + i_{pv}R_s)}{n_s AKT}}-1}
\end{equation}
where the photocurrent $I_{ph}$ and saturation current $I_s$ are functions of irradiance $G$ and cell temperature $T$:
\begin{align}
I_{ph} &= \qty{I_{sc} + K_i(T-T_r)}\frac{G}{G_r}\\
I_s&=I_{rs}\qty{\frac{T}{T_r}}^3\exp\qty{\qty{\frac{q E_g}{AK}}\qty{\frac{1}{T_r}-\frac{1}{T}}}\\
I_{rs}&=\frac{I_{sc}}{\exp\qty{\frac{q V_{oc}}{n_s K A T}}-1}
\end{align}
A detailed description of the parameters shown in these equations can be  found in Table \ref{tab: solar details}.
\begin{table}
\centering
\caption{Parameters and values of the solar example. Solar PV JAM72S30-545 specifications \citep{JAM72S30}}
\label{tab: solar details}
\begin{tabular}{l|c|l}
\toprule[1pt]
\textbf{Parameter}                                  & \textbf{Symbol}    & \textbf{Value}                   \\
\midrule
Capacitance and inductance of the DC/DC converter   & $C, L$             & $0.004~F$, $0.005~H$             \\
Battery capacitance                                 & $C_b$              & $1.8\times 10^5~F$               \\
Voltage source and resistance of the battery        & $E_b, R_b$         & $12~V$, $0.018~\Omega$            \\
Number of PV cells connected in parallel and series & $n_p,n_s$          & 1, 144                           \\
Electron charge                                     & $q$                & $1.6\times 10^{-19}~C$            \\
Series resistance                                   & $R_s$              & $0.05~\Omega$                     \\
Ideal factor of diode                               & $A$                & 1.3                              \\
Boltzmann constant                                  & $K$                & $1.38025\times 10^{-23}~J/K$      \\
Cell short-circuit current temperature coefficient  & $K_i$              & $0.045\%~A/K$                     \\
Reference temperature and irradiance                & $T_r,G_r$          & $298.15~K$, $1000~W/m^2$           \\
Cell open circuit voltage                           & $V_{oc}$           & $49.75~V$                        \\
Short circuit current                               & $I_{sc}$           & $13.93~A$                         \\
Band gap energy of the semiconductor                & $E_g$              & $1.1~eV$                          \\
\bottomrule[1pt]
\end{tabular}
\end{table}

In summary, we write the above equations as the following general continuous-time nonlinear system:
\begin{equation}\label{eq: solar dyn}
\vdot x(t) = F(\v x(t), u(t), \v \xi(t))
\end{equation}
where $\v x$ represents the state vector $\v x = [v_{pv}, i_s, v_c]^\T$; $u$ is the manipulated input (i.e., duty cycle) to the converter; $\v \xi = [G, T, i_L]^\T$ are external variables, that are independent of the state and the manipulated input.

\subsection{MPC Formulation}
In Eq. \eqref{eq: solar dyn}, the external variables $\v\xi$ are often provided by real-time sensors. However, we may only rely on some estimations of $\v \xi$ in MPC since it requires evaluations of the model at a few future time steps. To simplify the formulation of the optimization problem, we assume that the values of $\v \xi$ during a short period in the future are known in advance (i.e., $\v\xi(t_0),\dots,\v\xi(t_{N-1})$ at a regular time interval $\Delta=60~s$, for $t_n = t+n\Delta,\ n=0,1,\dots, N-1$). For example, the predictions of $G$ and $T$ can be obtained by developing a neural network to learn the pattern, e.g., an Input Convex Lipschitz RNN (ICLRNN) that has been developed in our previous work~\citep{wang2024input}. Moreover, a factory's daily base electricity consumption is more or less predictable and can be reliably forecasted for a short period. Flexible electricity consumption, much like weather forecasting, can be predicted using a suitable machine learning method with appropriate data collection.

To ensure the overall optimization is convex, we augment the state vector with an additional term that measures the absolute deviation in the current from the load, i.e., $d_t = |i_{s,t} - i_{L,t}|$. Denoting the augmented state vector as $\v z_t = \bmat{\v x_t \\ d_t}$, we train a neural network $F_{nn}$ to approximate Eq. \eqref{eq: solar dyn} in discrete-time:
\begin{equation}\label{eq: solar dyn discrete}
\vtilde z_{t+1} = F_{nn}(\v x_t, u_t, \v\xi_t)
\end{equation}
Note that in the solar PV example, all values in the augmented state are non-negative, and thus we apply the ReLU activation to the final outputs of our neural network.

In closed-loop simulations, we set the prediction horizon to two, and thus the decision variable is $\v u = [u_1, u_2]^\T$. The objective of the MPC is a sum of squared deviations of currents from the loads:
\begin{equation}
\mathcal{L}(\v u; \v x_t, \v\xi_t,\v\xi_{t+1}) = \qty{\tilde d_{t+1}}^2 + \qty{\tilde d_{t+2}}^2
\end{equation}
The value $\tilde d_{t+2}$ is obtained by feeding back the neural network with the predicted state $\vtilde x_{t+1}$ made by itself at the first step, i.e., $[{\vtilde x_{t+2}}^\T, \tilde d_{t+2}]^\T = F_{nn}(\vtilde x_{t+1}, u_2, \v\xi_{t+1})$.

The control actions are bounded by $u_\text{max}=0.95$ and $u_\text{min}=0.1$. Similar to \cite{qi2010supervisory}, we set constraints on: (1) the voltage of the battery $v_b = E_b + v_c + (i_s - i_L)R_b$ in $11.7~V\sim 14.7~V$, to avoid overcharging or complete drainage; (2) the magnitude of change in $i_s$, i.e., $|\tilde i_{s,t+1}-i_{s,t}| \leq \delta_\text{max}$ and $|\tilde i_{s,t+2} - i_{s,t}|\leq \delta_\text{max}$, where $\delta_\text{max}=8~A$; (3) $v_{pv}$ in $10~V \sim 60~V$ as the operating range.

\subsection{Process modeling}
To demonstrate the modeling behaviors of IC-LSTM, we trained two models, LSTM and IC-LSTM, each with a hidden layer configuration of $(64, 64)$. The inputs, $[\vt x_t, G_t, T_t, u_t, i_{L,t}]\in\R[7]$, are uniformly distributed in the rectangular domain with a lower bound $[10~V,0~A,0~V,10~W/m^2, $ $15~K, 0.1, 0~A]$ and a upper bound $[40~V, 20~A, 1~V, $ $1100~W/m^2, 70~K, 0.95, 25~A]$. To make the training data compatible with the recurrent model, each input is repeated $m$ times to become a sequence of length $m$, where $m$ is a positive integer. As mentioned in Section \ref{sec: iclstm}, we further expand the inputs to include the \textbf{negated values}. Thus, each training input is in $\R[m\times 14]$. The targets are trajectories of the augmented state $\v z_{t:t+\Delta}\in\R[m\times 4]$ in the interval $\Delta$ that are recorded at every $\Delta_s < \Delta$ time step, and such that $\Delta = m\Delta_s$. To improve training stability and model accuracy, we normalize the data with the \texttt{MinMaxScaler} method from SciPy \citep{virtanen2020scipy}. In the experiments, we set $m=10$, and the batch size of 128. All models are trained based on the mean squared error (MSE) loss, and the Adam optimizer with an initial learning rate of $0.001$ that will be halved when the testing loss is not decreasing. The final testing MSE of the LSTM reaches $1.62\times 10^{-4}$ while that of the IC-LSTM is $7.40\times 10^{-3}$.

The inferior performance of the IC-LSTM in modeling accuracy is  expected due to the process of convexification as discussed in Section \ref{sec: iclstm}. In Fig. \ref{fig: solar sim vs nn}, we plot the states in dashed lines by integrating the first-principles model (Eq. \eqref{eq: solar dyn}) for one sampling time $\Delta=60~s$, starting with the same initial conditions, $v_{pv}=25~V, i_s=5~A, v_c=1~V, G=500~W/m^2, T= 55~^{\circ} C$ and $i_L=7~A$, and different values of the duty cycle $u\in[0.1, 0.95]$. The solid lines are obtained by calling the trained models as in Eq. \eqref{eq: solar dyn discrete}. It is readily shown that for the IC-LSTM, the outputs are convex functions of the control action. However, local features of the target system are inevitably lost in a global convex representation (this is the drawback for any input-convex neural networks), and therefore, IC-LSTM cannot fully capture the peak in the output current and the narrow basin in the absolute deviation. While the modeling accuracy of IC-LSTM is reduced due to convexity, we will demonstrate in the next subsection on control performance that the degradation in modeling accuracy does not significantly impact closed-loop performance.

\begin{figure}[ht!]
    \centering
    \begin{subfigure}[t]{\textwidth}
        \centering
        \includegraphics[width=\columnwidth]{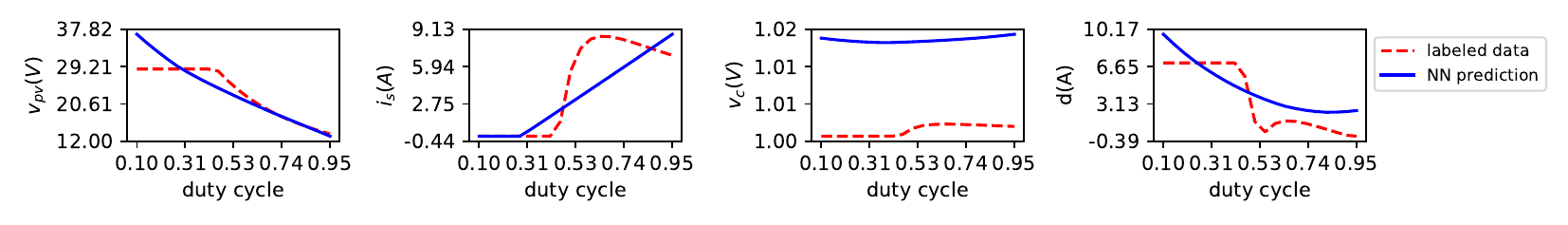}
        \caption{IC-LSTM}
    \end{subfigure}
    ~
    \begin{subfigure}[t]{\textwidth}
        \centering
        \includegraphics[width=\columnwidth]{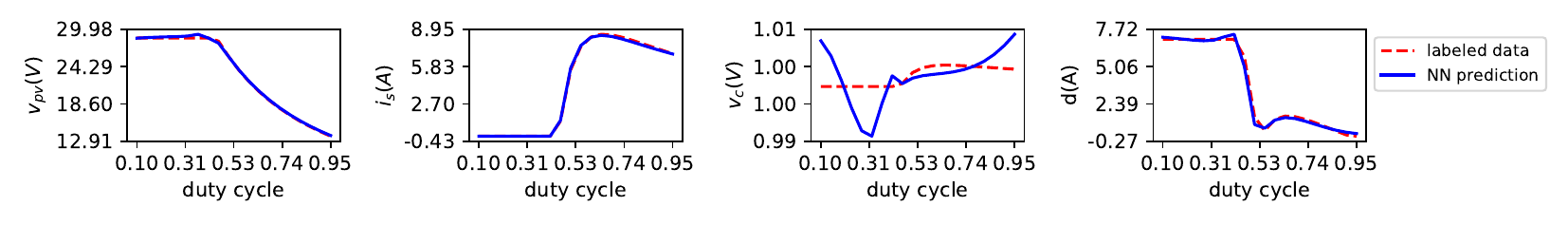}
        \caption{LSTM}
    \end{subfigure}
    \caption{Labeled data (dashed lines) and neural network predictions (solid lines). Along the horizontal axis, as the value of the duty cycle varies from $0.1$ to $0.95$, we keep values of other parameters unchanged, i.e., $v_{pv}=25~V, i_s=5~A, v_c=1~V, G=500~W/m^2, T=55~^{\circ} C, i_L=7~A$. To obtain the labeled data, we integrate the first-principles model for one sampling time $\Delta=60~s$ and take the final values of the state.}
    \label{fig: solar sim vs nn}
\end{figure}

\begin{remark}
A simple but effective method to assess whether the model effectively learns a convex representation from a general nonlinear dynamics is by examining the training and testing MSE. Ideally, a successful input convex model should achieve a moderate MSE (i.e., neither as low as conventional non-convex models nor as high as randomly initialized models). Based on our findings, a model that successfully learns a convex representation typically achieves an MSE in the range of $10^{-3}$ to $10^{-4}$ for normalized data. 
\end{remark}

\subsection{Control Performance}
To validate the control performance, we optimize the aforementioned solar PV energy system using the real-world data of the solar irradiance and temperature of the solar PV panel \citep{JAM72S30} installed at LHT Holdings. In this experiment, the optimization problem was solved using the PyIpopt library. We highlight that the findings in the results are consistent throughout the year, as Singapore's weather remains generally stable due to its equatorial location. For demonstration purposes, we used the data on May 5, 2024 (see Fig. \ref{fig: si real data}), and ran the MPC at two different time windows (i.e., 10 a.m. and 1 p.m.) with randomly generated loads $i_L$. The values of solar irradiance, panel temperature, and the load current remain constant during the sampling period $\Delta=60~s$.

Table \ref{tab: solar runtime} shows the average solving times (over 5 random runs) per step for LSTM and IC-LSTM. In all cases, IC-LSTM enjoys a faster (at least $4 \times$) solving time (for a scaled-up solar PV energy system or a longer prediction horizon, the time discrepancy could be even greater). However, due to the modeling errors (i.e., especially for $i_s$ and $d$ as shown in Fig. \ref{fig: solar sim vs nn}), it is difficult for IC-LSTM to closely track the demand currents. Specifically, we show the tracking results of both models for one run at 10 a.m. in Fig. \ref{fig: solar mpc}. In the top right subplots, the dashed lines are the demand loads, and the solid lines are the real output currents of the system based on the first-principles model and the optimal values of the manipulated input $u$ from solving the MPC problem. While the modeling accuracy of LSTM is superior, the tracking performances of both models are acceptable in practice. From the plot of duty cycles (middle right subplots), we observe that the values of IC-LSTM are consistently larger, i.e., towards the upper bound $u_\text{max}=0.95$, whereas the values of LSTM have a larger variance and a mean at $0.5$. The bias in the calculated duty cycle of IC-LSTM can be explained by referring to Fig. \ref{fig: solar sim vs nn}, as it is favorable for IC-LSTM to operate in the high duty cycle region where the output current $i_s$ monotonically increases and the absolute current deviation $d$ is small (i.e., at bottom of the basin). However, a better performance is often achieved around $u\approx 0.5$, where the output current peaks and is more sensitive to the manipulated input. Note that in Fig. \ref{fig: solar mpc}, both models experience a drop in $i_s$ from the $7^{th}$ minute onward. This corresponds to the drop in the solar irradiance and hence it is infeasible for the solar PV system to attain the load. Note that since the photocurrent of a cell is proportional to irradiance, perfect tracking may be infeasible under low illuminations.

\begin{figure}[ht!]
\centering
\includegraphics[width=\textwidth]{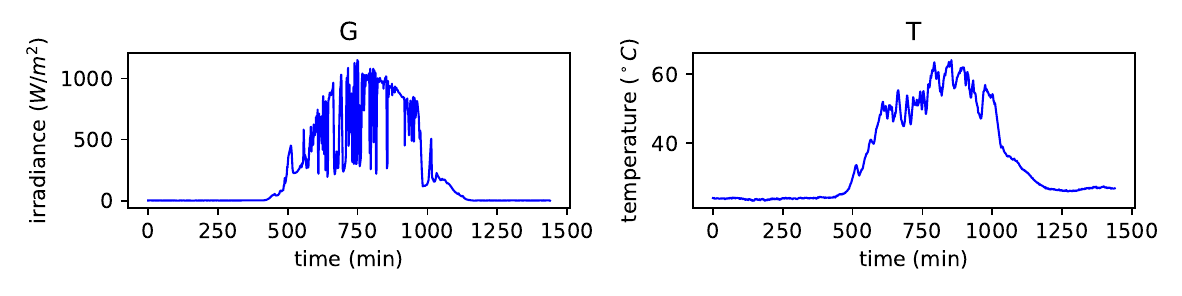}
\caption{Real-world data of the solar irradiance and temperature of the solar PV panel recorded on a minute-by-minute basis, May 5, 2024, LHT Holdings.}
\label{fig: si real data}
\end{figure}

\begin{table}[ht!]
\centering
\caption{Comparisons of average computational time (seconds)}
\vspace{.5em}
\label{tab: solar runtime}
\begin{tabular}{l|c|c}
\toprule[1pt]
\textbf{Time} & \textbf{IC-LSTM} & \textbf{LSTM}   \\
\midrule
10 a.m.         & $7.74\pm2.19$  & $32.01\pm7.83$ \\
1 p.m.          & $1.36\pm0.50$   & $31.60\pm9.71$  \\
\bottomrule[1pt]
\end{tabular}
\end{table}

\begin{figure}[ht!]
    \centering
    \begin{subfigure}[t]{0.48\textwidth}
        \centering
        \includegraphics[width=\columnwidth]{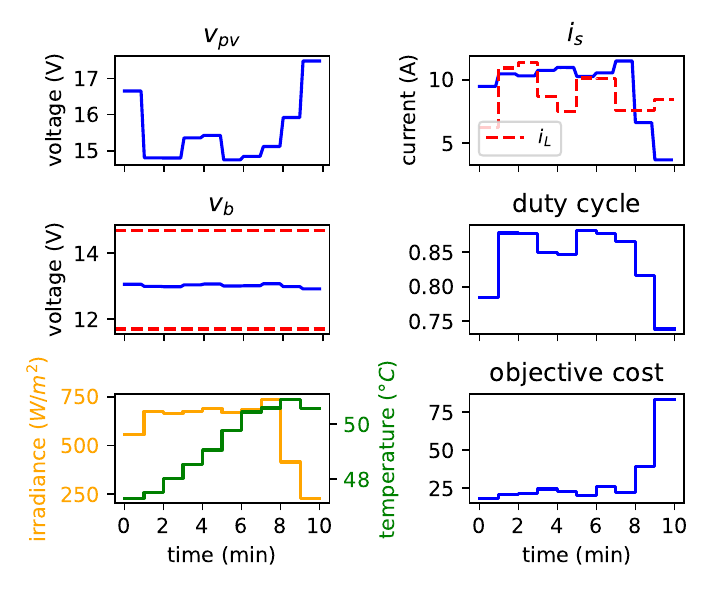}
        \caption{IC-LSTM}
    \end{subfigure}
    ~ 
    \begin{subfigure}[t]{0.48\textwidth}
        \centering
        \includegraphics[width=\columnwidth]{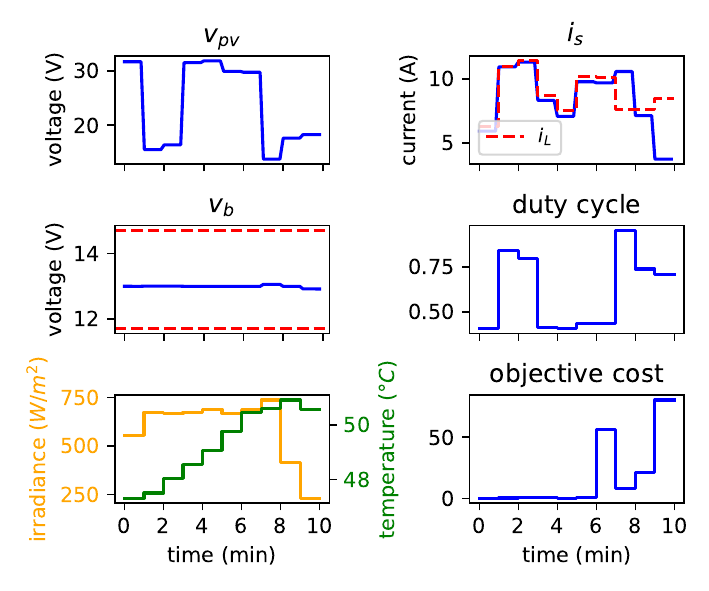}
        \caption{LSTM}
    \end{subfigure}
    \caption{Comparisons of MPC performances using IC-LSTM and LSTM. Real irradiance and temperature data were recorded on May 5, 2024, 10 a.m. at LHT Holdings. Dashed lines represent the demand currents $i_L$ in the top right subplots; the upper and lower bounds of $v_b$ in the middle left subplots.}
    \label{fig: solar mpc}
\end{figure}

Nonetheless, it is important to note that for a practical hybrid energy system (e.g., Fig. \ref{fig_energy_system}), the MPC does not need to produce a perfect tracking of the current. In the optimal situation where $i_s$ equals $i_L$, the loads can be fully satisfied by the solar PV system alone, without the need for input from the battery or the power grid. However, any discrepancy in the current can be managed by the factory drawing power from the main grid and the battery or storing excess energy in the battery. This discrepancy results in additional operational costs for the company, and it may be balanced by the computational speed gains offered by the input-convex model to operate in real-time.

\section{Application to a Chemical Process}
\label{sec4}
\subsection{System Description}
In this case study, we consider a real-time optimization-based control of a well-mixed, nonisothermal continuously stirred tank reactor, where an irreversible second-order exothermic reaction takes place (i.e., the reaction will transform a reactant $A$ to a product $B$). The CSTR is equipped with a heating jacket that supplies/removes heat at a rate of $Q$. The CSTR dynamic model is described by the following material and energy balance equations:
\begin{subequations}
\begin{align}
\dv{C_A}{t} &= \frac{F}{V_L}(C_{A0}-C_A)-k_0e^{\frac{-E}{RT}}C_A^2 \\
\dv{T}{t} &= \frac{F}{V_L}(T_0-T)+\frac{-\Delta H}{\rho_LC_p}k_0e^{\frac{-E}{RT}}C_A^2+\frac{Q}{\rho_LC_pV_L}
\end{align} 
\label{eqcstr}
\end{subequations}
where $C_A$ is the concentration of reactant $A$, $T$ is the temperature, $Q$ is the heat input rate, and $C_{A0}$ is the inlet concentration of reactant $A$. The remaining parameters and their values are shown in Table \ref{tab_cstr_values}.

\begin{table}
\centering
\caption{Parameters and values of CSTR}
\label{tab_cstr_values}
\begin{tabular}{l|c|l}
\toprule[1pt]
\textbf{Parameter}                                  & \textbf{Symbol}    & \textbf{Value}                   \\
\midrule
Volumetric flow rate & $F$ & $5~m^3/hr$ \\
Volume of the reacting liquid & $V_L$ & $1~m^3$ \\
Ideal gas constant & $R$ & $8.314~kJ/kmol~K$ \\
Inlet temperature & $T_0$ & $300~K$ \\
Heat capacity & $C_p$ & $0.231~kJ/kg~K$ \\
Constant density of the reacting liquid & $\rho_L$ & $1000~kg/m^3$ \\
Activation energy & $E$ & $5 \times 10^4~kJ/kmol$ \\
Pre-exponential constant & $k_0$ & $8.46 \times 10^6~m^3/kmol~hr$ \\
Steady-state heat input rate & $Q_s$ & $0.0~kJ/hr$ \\
Steady-state inlet concentration of reactant A & $C_{A0_s}$ & $4~kmol/m^3$\\
Equilibrium concentration of reactant A & $C_{A_s}$ & $1.95~kmol/m^3$\\
Equilibrium temperature & $T_s$ & $402~K$\\
Enthalpy of reaction & $\Delta H$ & $-1.15 \times 10^4~kJ/kmol$ \\
\bottomrule[1pt]
\end{tabular}
\end{table}

The manipulated input vector is $\v u = [\Delta C_{A0}, \Delta Q]^\T$, where $\Delta C_{A0} = C_{A0} - C_{A0_s}$ and $\Delta Q = Q - Q_s$ are the deviations (from the steady state values) of the inlet concentration of reactant $A$ and the heat input rate respectively. The state vector of the CSTR system is $\v x = [C_A - C_{As}, T - T_s]^\T$, and thus the equilibrium point of the system is located at the origin of the state-space. In training, the inputs of the neural network consist of $[T_t - T_s, C_{A,t} - C_{As}, \Delta Q_t, \Delta C_{A0,t}]$ at the current time step $t$, along with their respective negations (as \textbf{expanded inputs}). The outputs of the neural network entail the \textbf{absolute values} of state trajectory $[T_{t+1:n} - T_s, C_{A,t+1:n} - C_{As}]$ over the subsequent $n$ time steps (i.e., representing a sequence problem). Moreover, the main control objective is to operate the CSTR at the unstable equilibrium point $(C_{As}, T_s)$ by manipulating $\Delta C_{A0}$ and $\Delta Q$, using the MPC in Eq. \eqref{eq9} with an additional Lyapunov-Based constraint shown in Eq. \eqref{eq9e} with neural networks, and finally reach the steady state.

\subsection{CSTR Modeling}
We constructed and trained neural networks, which are meticulously configured with a batch size of 256, the Adam optimizer, and the MSE loss function. The dataset is generated from computer simulations following the method in \cite{wu2019machine2} (using forward Euler method). Note that the proposed IC-LSTM modeling method is not limited to simulation data. It can be effectively applied to various data sources, including experimental data and real-world operational data.

The model designs are shown in Table \ref{tabhyper}, including the number of floating point operations (FLOPs). The primary purpose of neural networks is to capture and encapsulate the system dynamics, subsequently integrating into the Lyapunov-based MPC (LMPC) framework shown in Eq. \eqref{eq9} with Eq. \eqref{eq9e}. In the CSTR task, the LMPC aims to drive the system from its initial condition to the steady-state (convergence). Based on our previous works \citep{wu2019machine1, wu2019machine2}, we observed that model performance, with an MSE on the order of $\mathbf{10^{-3}}$ for \textbf{normalized data}, is sufficient to ensure convergence. Therefore, further increasing modeling accuracy may not be necessary if it significantly increases computational costs for neural network training and solving MPC optimization problems. Additionally, the model structures are designed to be minimally complex while achieving the desired performance. As shown in Table \ref{tabhyper}, although ICRNN and IC-LSTM exhibit some accuracy deficiencies (as expected), these models are sufficient to stabilize the system when incorporated into MPC.

\begin{table}[htbp]
\centering
\caption{Hyperparameters of neural network models}
\vspace{.5em}
\resizebox{\linewidth}{!} {
\begin{tabular}{c|c|c|c|c|c|c}
\hline
\textbf{Model} & \textbf{Activation} & \textbf{No. of Layers} & \textbf{No. of Neurons} & \textbf{Test MSE} & \textbf{No. of Parameters} & \textbf{FLOPs}\\
\hline
Plain RNN & Tanh & 2 & 64 & $ 3.53\times 10^{-5} \pm 3.85 \times 10^{-6}$ & 12,802 & 27,924\\
Plain LSTM & Tanh & 2 & 64 & $2.61 \times 10^{-6} \pm 1.90 \times 10^{-7}$ & 50,818 & 104,468\\
ICRNN & ReLU & 2 & 64 & $ 5.50\times 10^{-4} \pm 1.04 \times 10^{-4}$ & 38,530 & 78,602\\
IC-LSTM (Ours) & ReLU & 2 & 64 & $4.12 \times 10^{-3} \pm 7.85 \times 10^{-6}$ & 11,298 & 86,858\\
\hline
\end{tabular}}
\label{tabhyper}
\end{table}


\subsection{Lyapunov-Based MPC Formulation}
\label{subsec_lmpc}
In the closed-loop control task, we assume that there exists a stabilizing controller $u = \Phi(x) \in U$ that renders the equilibrium point defined by Eq. \eqref{eq:classofsystems} asymptotically stable. Thus, in addition to the MPC designed in Eq. \eqref{eq9}, we introduce an additional Lyapunov-based constraint to form a LMPC, as follows:
\begin{equation}
\label{eq9e}
    V({\vtilde{x}}(t)) < V(\v{x}(t_k)),\text{ if }\v x(t_k) \in \Omega_{\rho} \backslash \Omega_{\rho_{nn}}, \ \forall t \in [t_k,t_{k+N})
\end{equation}
where $\Omega_{\rho}$ is the closed-loop stability region of the system, and $\Omega_{\rho_{nn}}$ is a small set around the origin where the state should ultimately be driven. The Lyapunov-based constraint $V$ ensures closed-loop stability for the nonlinear system under LMPC by requiring that the value of $V(\v x)$ decreases over time.
Let us define the control Lyapunov function as $V(\v x) = \vt x\v P\v x$, where $\v x \in\R[n_x]$ and $\v P \in \R[n_x \times n_x]$. For $V(\v x)$ to show convexity, it is necessary for its Hessian matrix, denoted as $\v H$, to be positive semidefinite. In this case, $\v H$ is equal to $2\v P$. Therefore, Eq. \eqref{eq9e}, represented as $V(\vtilde{x}(t)) - V(\v x(t)) < 0$, is convex without loss of generality, by selecting a positive semidefinite matrix $\v P$.

Note that the state vector is in the deviation form and the origin of the state-space is the equilibrium point at where the LMPC is targeting. To ensure that the overall optimization problem is convex with respect to the manipulated vector, we let the IC-LSTM to learn the absolute values of the state vector. We define the objective function as $\norm{\vtilde{x}}^2+\norm{\v u}^2$, which has the global minimum at the steady state $(\v 0, \v 0)$. Given that Eq. \eqref{eq9b} is parameterized as IC-LSTM, and both Eq. \eqref{eq9c} and Eq. \eqref{eq9d} take the form of affine functions, the LMPC problem outlined in Eq. \eqref{eq9} with Eq. \eqref{eq9e} qualifies as a convex optimization problem, provided that the $\v P$ matrix is designed to be positive semi-definite.

The LMPC problem is solved using PyIpopt, which is the Python version of IPOPT \citep{wachter2006implementation}, with an integration time step of $h_c = 1 \times 10^{-4}~hr$ and the sampling period $\Delta = 5 \times 10^{-3}~hr$. The control Lyapunov function $V(\v x) = \vt x\v P\v x$ is designed with the following positive definite $\v P$ matrix as $\bmat{1060 & 22 \\ 22 & 0.52}$, which ensures the convexity of the LMPC. Moreover, the equation for the stability region is defined as $1060x^2 + 44xy + 0.52y^2 - 372 = 0$, which is an ellipse in the state space.

\subsubsection{Control Performance}
The efficacy of LMPC is based on two critical factors: the ability to reach a steady state and the time required for convergence (solving time). To address this, we emphasize the temporal aspect by evaluating the time it takes for the neural network-based LMPC to achieve stability. For the CSTR example, we define the small region of stability to be $|C_A - C_{As}| < 0.02~kmol/m^3$ and $|T - T_s| < 3~K$ and the system is considered practically stable only when both conditions are met simultaneously (i.e., the program will terminate immediately upon system convergence). Moreover, the computational time (i.e., convergence runtime) that drives the system from the initial state to the steady state and the system state after each iteration will be recorded.

Our experimentation focuses on evaluating the control performance of the CSTR of Eq. \eqref{eqcstr} by embedding IC-LSTM into the LMPC of Eq. \eqref{eq9} with the stability constraint of Eq. \eqref{eq9e}. In this experiment, the PyIpopt library was executed on an Intel Core i7-12700 processor with 64 GB of RAM, using 15 different initial conditions within the stability region (i.e., covering the whole stability region). In particular, all trials successfully achieved convergence to the steady state (e.g., Fig. \ref{fig_path} shows the convergence paths of neural network-based LMPC of two initial conditions for demonstration purposes). Specifically, in Fig. \ref{fig_path}, we also perform the simulation under the LMPC using the first-principles model for the CSTR, as the benchmark case for comparison. The first-principles-model-based LMPC demonstrates the smoothest convergence path, with fewer turning points and less fluctuation near the steady state. In contrast, RNN-LMPC and LSTM-LMPC show more turning points and greater fluctuations as they approach the steady state. Among various neural network approaches, IC-LSTM shows the smoothest convergence path, and closely resembles the trajectory of first-principles model-based LMPC. However, it is important to note that similar convergence paths do not necessarily correspond to similar solving times. Observing only the convergence paths does not provide insight into the solving time. Although the neural network-based LMPC exhibits similar convergence paths to the first-principles-model-based LMPC, the solving times are significantly different: the RNN-LMPC takes over 1,500 seconds, the IC-LSTM-LMPC takes over 700 seconds, while the first-principles-based LMPC requires only around 2 seconds. Moreover, Fig. \ref{fig_time} shows the state trajectories of $C_A - C_{As}$ and $T - T_s$ over time in seconds for two initial conditions, demonstrating that the IC-LSTM-based LMPC achieves the fastest convergence. It is important to note that while different neural network-based LMPC methods may have a similar number of steps or iterations to reach a steady state, their computational times can vary significantly. 

Furthermore, Table \ref{tab5} presents the average solving time in 3 random runs for 15 initial conditions and their corresponding percentage decrease with respect to IC-LSTM, demonstrating that the IC-LSTM-based LMPC shortens computational time (i.e., the IC-LSTM-based LMPC achieved the fastest convergence in 13 out of 15 different initial conditions). Specifically, it achieves an average percentage decrease of 54.4\%, 40.0\%, and 41.3\% compared to plain RNN, plain LSTM, and ICRNN, respectively. Overall, the ICRNN performs similarly to the plain LSTM in this optimization task, while the plain RNN performs the worst. In summary, the optimization performance, in terms of convergence path, is comparable between models with MSE on the order of $10^{-3}$ (IC-LSTM) and those with MSE on the order of $10^{-6}$ (LSTM) (see Fig. \ref{fig_path}), while the solving time differs significantly (see Table \ref{tab5}).

\begin{figure}[ht!]
    \centering
    \begin{subfigure}[t]{0.48\textwidth}
        \centering
        \includegraphics[width=\columnwidth]{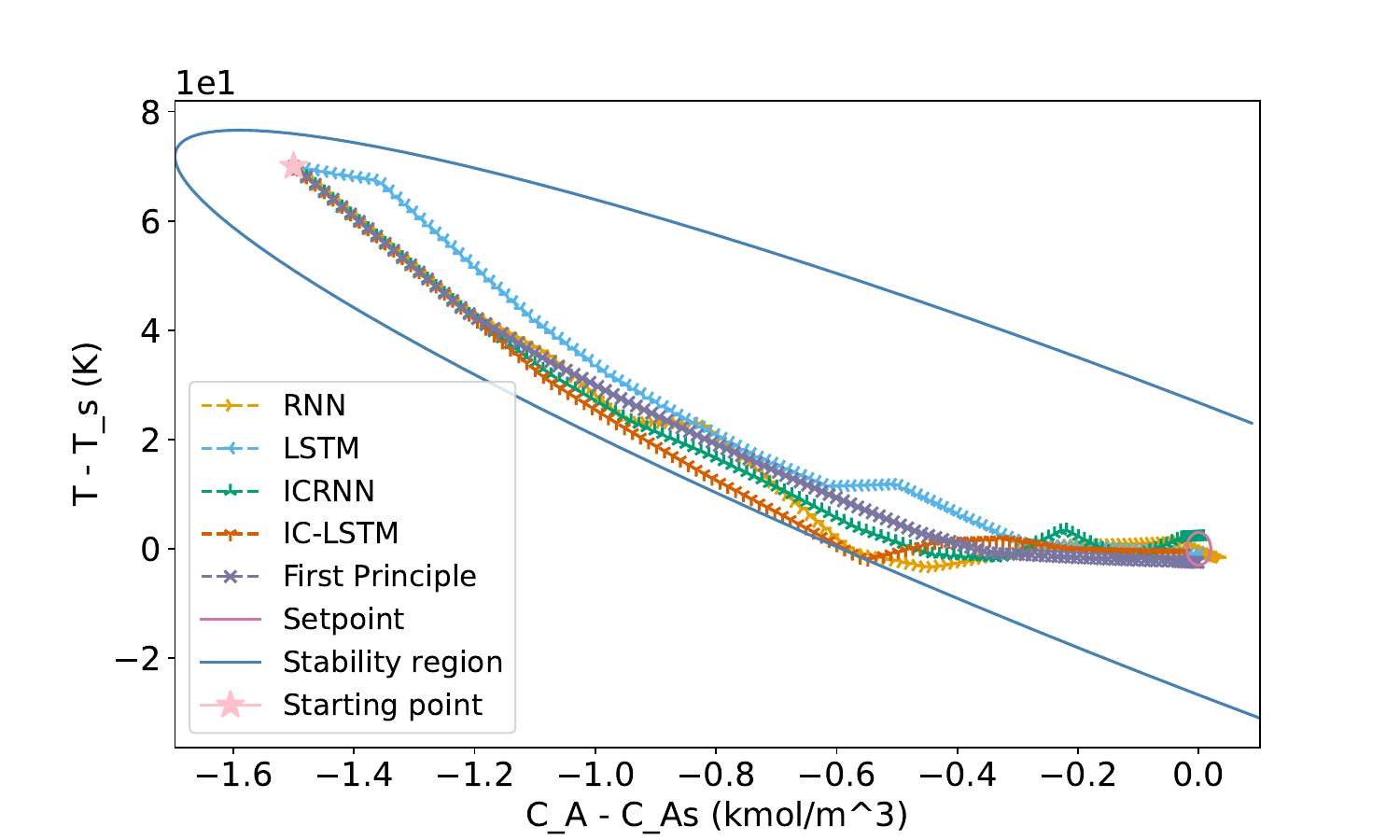}
        \caption{Initial condition of ($-1.5~kmol/m^3, 70~K$)}
        \label{fig_path70}
    \end{subfigure}
    ~ 
    \begin{subfigure}[t]{0.48\textwidth}
        \centering
        \includegraphics[width=\columnwidth]{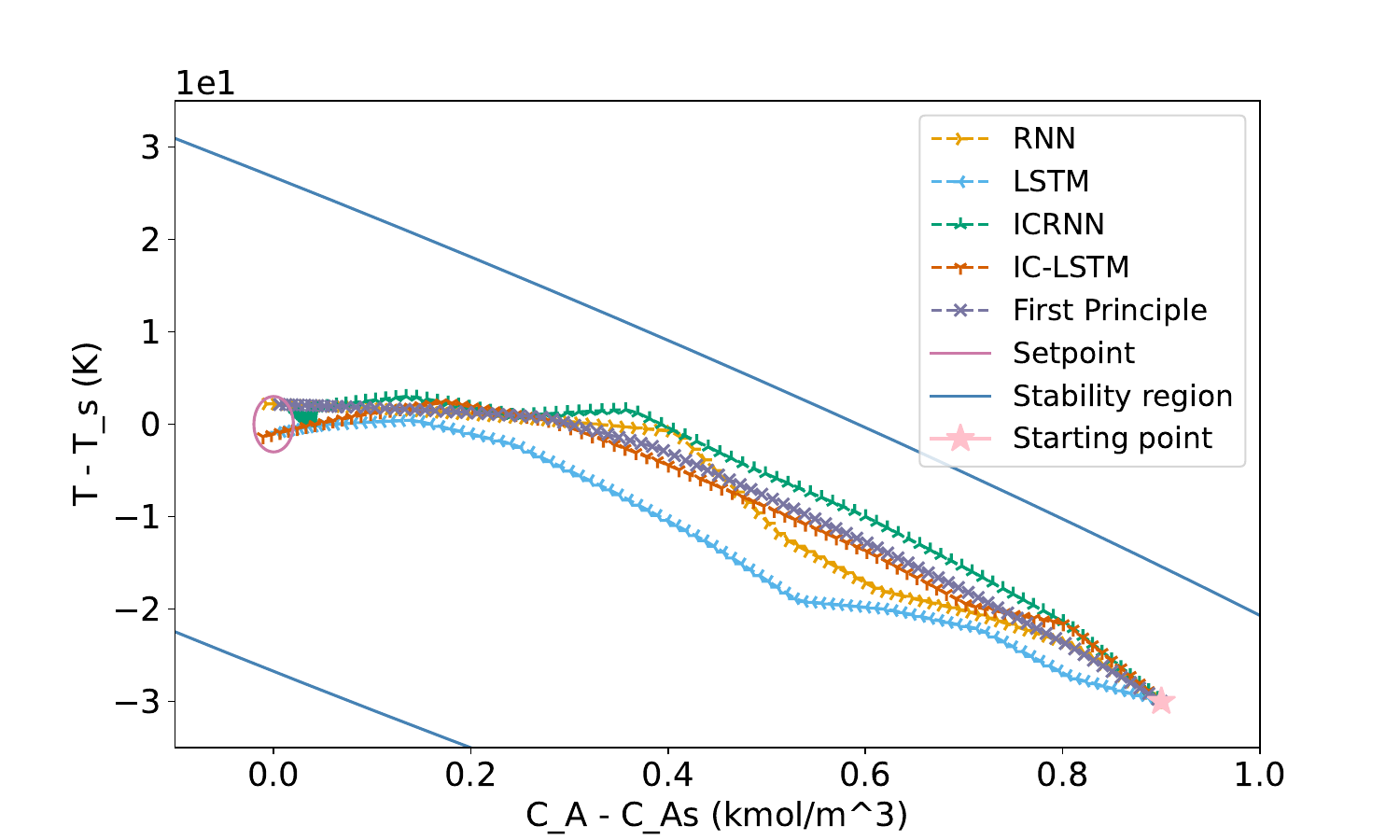}
        \caption{Initial condition of ($0.9~kmol/m^3, -30~K$)}
        \label{fig_path30}
    \end{subfigure}
    \caption{Closed-loop state trajectory ($T - T_s$ vs. $C_A - C_{As}$) under neural network-based MPC.}
    \label{fig_path}
\end{figure}

\begin{figure}[ht!]
    \centering
    \begin{subfigure}[t]{0.48\textwidth}
        \centering
        \includegraphics[width=\columnwidth]{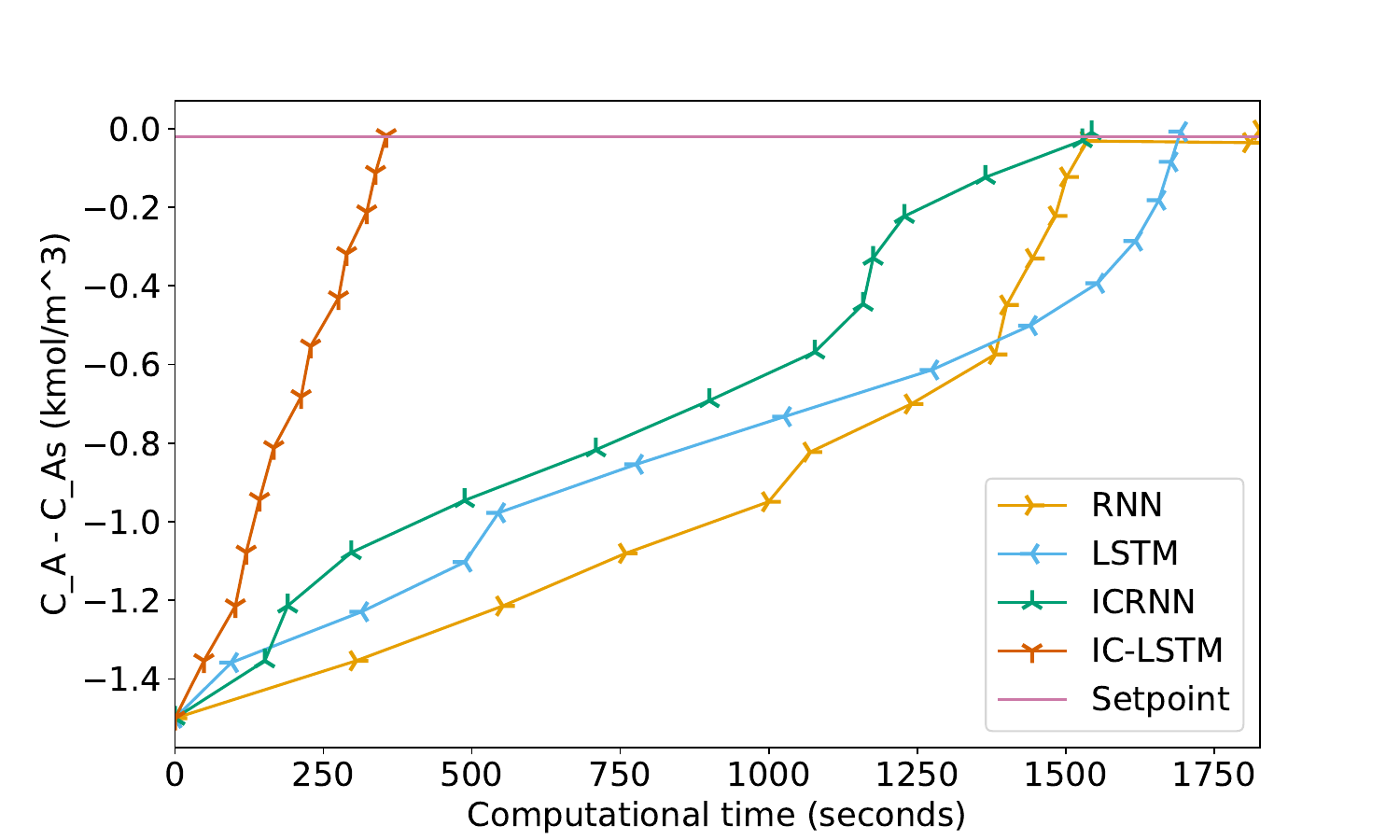}
        \caption{$C_A - C_{As}$ vs. computational time for initial condition of ($-1.5~kmol/m^3, 70~K$)}
        \label{fig_con70}
    \end{subfigure}
    ~ 
    \begin{subfigure}[t]{0.48\textwidth}
        \centering
        \includegraphics[width=\columnwidth]{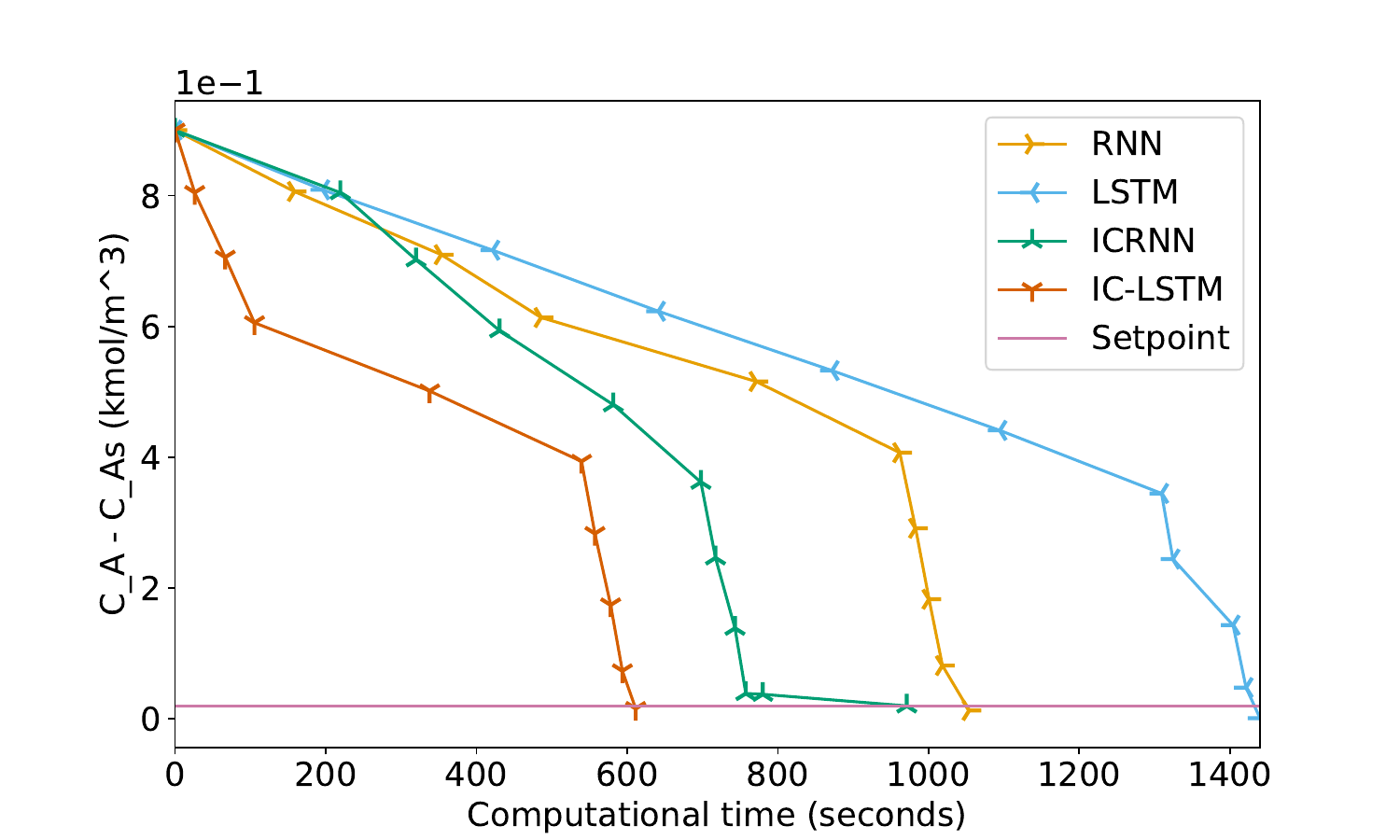}
        \caption{$C_A - C_{As}$ vs. computational time for initial condition of ($0.9~kmol/m^3, -30~K$)}
        \label{fig_con30}
    \end{subfigure}
    ~
    \begin{subfigure}[t]{0.48\textwidth}
        \centering
        \includegraphics[width=\columnwidth]{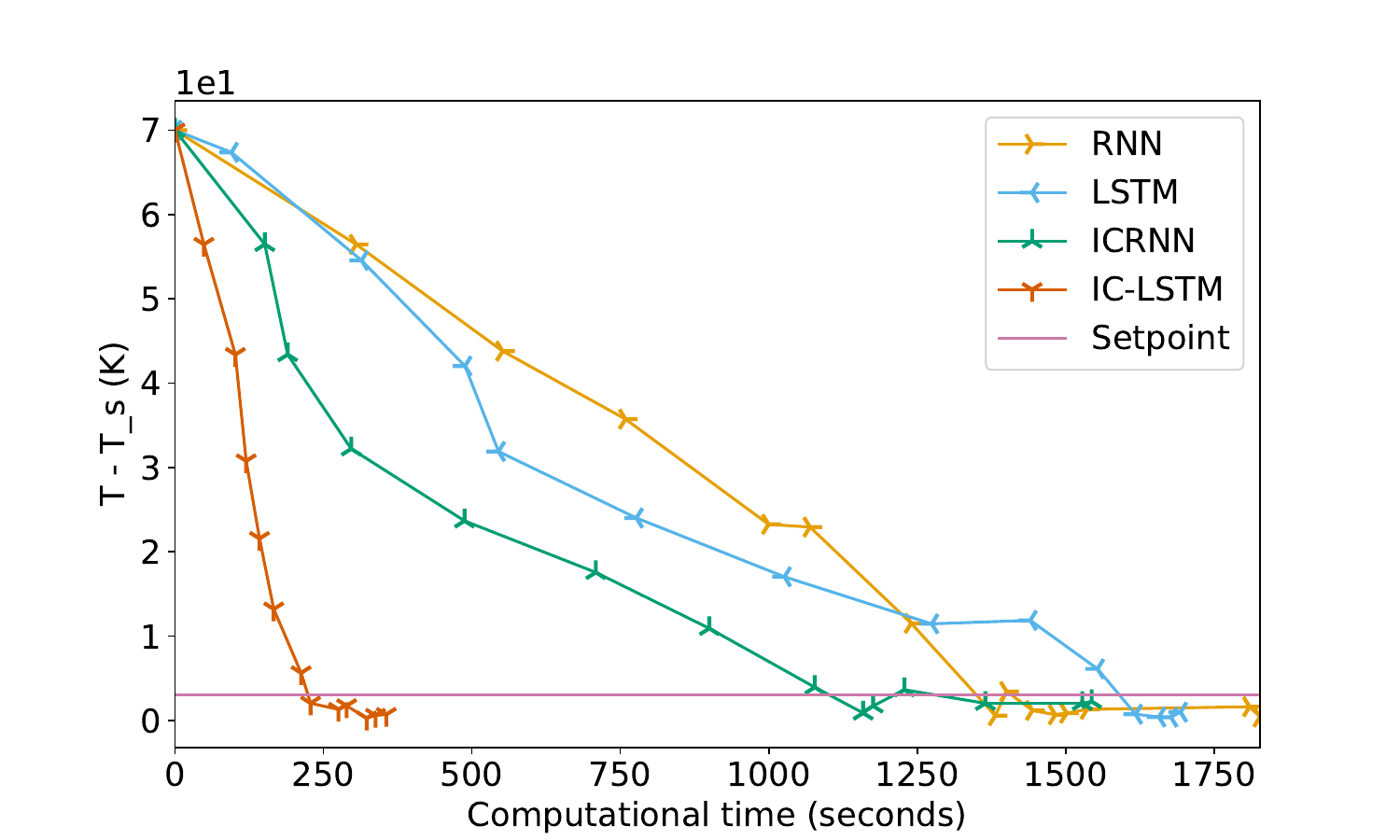}
        \caption{$T - T_s$ vs. computational time for initial condition of ($-1.5~kmol/m^3, 70~K$)}
        \label{fig_temp70}
    \end{subfigure}
    ~ 
    \begin{subfigure}[t]{0.48\textwidth}
        \centering
        \includegraphics[width=\columnwidth]{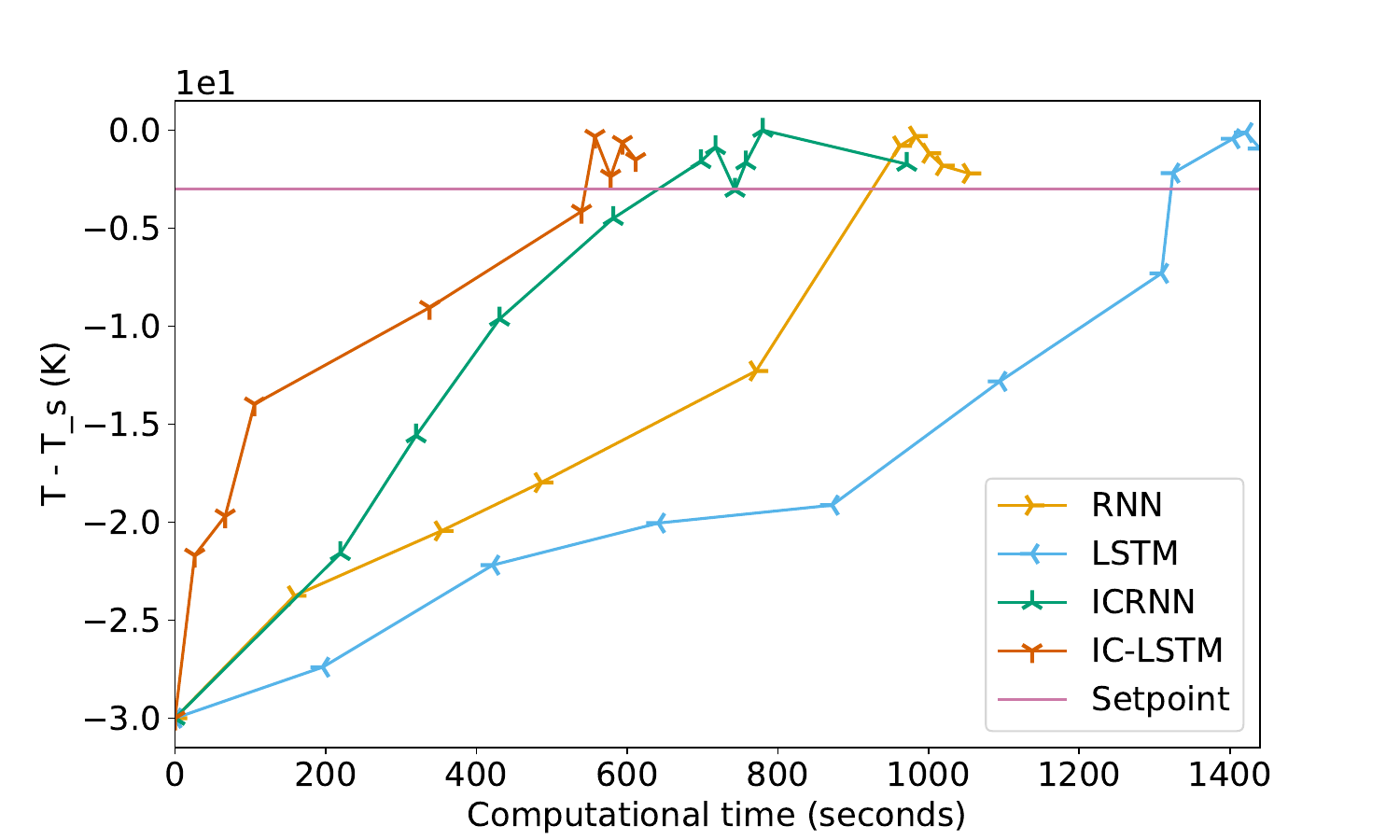}
        \caption{$T - T_s$ vs. computational time for initial condition of ($0.9~kmol/m^3, -30~K$)}
        \label{fig_temp30}
    \end{subfigure}
    \caption{State trajectories of neural network-based LMPC with respect to computational time.}
    \label{fig_time}
\end{figure}

\begin{table}[htbp]
\centering
\vspace{.5em}
\caption{Computational time of neural network-based LMPC and their respective percentage decrease with respect to IC-LSTM-based LMPC}
\resizebox{\linewidth}{!} {
\begin{tabular}{c|c|c|c|c|c|c|c}
\hline
& \multicolumn{2}{|c|}{\textbf{Plain RNN}} & \multicolumn{2}{|c|}{\textbf{Plain LSTM}} & \multicolumn{2}{|c|}{\textbf{ICRNN}} & \textbf{IC-LSTM} (Ours) \\
\cline{2-8}
$\mathbf{[C_{A_i}, T_i]}$& \textbf{Time (s)} & \textbf{\% Decrease} & \textbf{Time (s)} & \textbf{\% Decrease} & \textbf{Time (s)} & \textbf{\% Decrease} & \textbf{Time (s)} \\
\hline
$[-1.5, 70]$ & $1815.98 \pm 8.17$ & $79.62\%$ & $1688.68 \pm 3.40$ & $78.08\%$ & $1550.70 \pm 5.74$ & $76.13\%$ & $\mathbf{370.17} \pm 11.22$ \\
$[-1.3, 60]$ & $1382.14 \pm 9.48$ & $59.21\%$ & $1632.31 \pm 7.05$ & $65.46\%$ & $1387.31 \pm 8.46$ & $59.37\%$ & $\mathbf{563.72} \pm 17.80$ \\
$[-1, 55]$ & $1552.00 \pm 8.38$ & $71.95\%$ & $1391.79 \pm 3.36$ & $68.73\%$ & $1384.69 \pm 9.15$ & $68.57\%$ & $\mathbf{435.26} \pm 4.08$ \\
$[-1.25, 50]$ & $1283.54 \pm 10.83$ & $64.54\%$ & $1453.57 \pm 27.28$ & $68.69\%$ & $1291.00 \pm 13.59$ & $64.75\%$ & $\mathbf{455.10} \pm 3.58$ \\
$[-0.75, 40]$ & $1955.55 \pm 10.02$ & $60.40\%$ & $1079.38 \pm 15.65$ & $28.27\%$ & $1202.08 \pm 11.31$ & $35.59\%$ & $\mathbf{774.26} \pm 4.90$ \\
$[-0.5, 30]$ & $961.90 \pm 6.94$ & $55.58\%$ & $764.51 \pm 14.80$ & $44.11\%$ & $829.13 \pm 23.25$ & $48.47\%$ & $\mathbf{427.26} \pm 2.87$ \\
$[-0.45, 15]$ & $485.19 \pm 1.66$ & $9.33\%$ & $757.02 \pm 13.81$ & $41.89\%$ & $1174.15 \pm 16.13$ & $62.53\%$ & $\mathbf{439.91} \pm 5.89$ \\
$[1.5, -70]$ & $3937.33 \pm 84.67$ & $64.88\%$ & $1556.54 \pm 39.61$ & $11.16\%$ & $1731.51 \pm 12.11$ & $20.13\%$ & $\mathbf{1382.98} \pm 1.48$ \\
$[1.35, -55]$ & $2513.77 \pm 119.03$ & $34.00\%$ & $\mathbf{1472.01} \pm 11.43$ & $-12.71\%$ & $1640.15 \pm 10.92$ & $-1.16\%$ & $1659.11 \pm 35.60$ \\
$[1.1, -45]$ & $1394.29 \pm 4.55$ & $24.60\%$ & $1099.92 \pm 17.66$ & $4.42\%$ & $1236.44 \pm 27.80$ & $14.97\%$ & $\mathbf{1051.35} \pm 36.47$ \\
$[0.9, -30]$ & $1060.22 \pm 4.46$ & $44.62\%$ & $1453.77 \pm 22.88$ & $59.61\%$ & $977.75 \pm 14.97$ & $39.95\%$ & $\mathbf{587.12} \pm 17.25$ \\
$[0.75, -40]$ & $1957.87 \pm 8.30$ & $52.12\%$ & $1030.41 \pm 9.11$ & $9.03\%$ & $1460.97 \pm 54.15$ & $35.84\%$ & $\mathbf{937.35} \pm 28.40$ \\
$[0.6, -25]$ & $1646.76 \pm 52.49$ & $54.36\%$ & $927.37 \pm 33.98$ & $18.96\%$ & $1018.70 \pm 13.36$ & $26.22\%$ & $\mathbf{751.56} \pm 8.47$ \\
$[0.4, -35]$ & $1203.42 \pm 41.60$ & $61.24\%$ & $765.72 \pm 20.85$ & $39.08\%$ & $533.06 \pm 21.60$ & $12.49\%$ & $\mathbf{466.47} \pm 8.45$ \\
$[0.2, -15]$ & $\mathbf{263.55} \pm 4.74$ & $-46.64\%$ & $725.71 \pm 19.64$ & $46.74\%$ & $797.74 \pm 4.74$ & $51.55\%$ & $386.48 \pm 5.55$ \\
\hline
\textbf{Average} & $1560.9$ & $54.4\%$ & $1186.6$ & $40.0\%$ & $1214.4$ & $41.3\%$ & $\mathbf{712.5}$ \\
\hline
\end{tabular}
}
\label{tab5}
\end{table}

\begin{remark}
\label{remark7}
It is important to note that input convex models may not exhibit the same level of performance as conventional non-convex machine learning models. This discrepancy arises from the smoothing effect on non-convex features in the data, a process known as convexification. Despite this limitation, an input convex structure proves advantageous in optimization problems. In practice, users are encouraged to carefully assess the advantages and disadvantages of employing an input convex structure, taking into account their specific goals and requirements. In summary, employing IC-LSTM involves a trade-off between computational efficiency in solving neural network-based optimization and modeling accuracy. It is advisable to use IC-LSTM in optimization contexts where computational speed is crucial (e.g., in real-time optimization and control problems).
\end{remark}

\begin{remark}
In addition to the black-box modeling approach (i.e., neural network modeling discussed in this work), there are also gray-box modeling methods, such as those described in \cite{bunning2022physics} and parameter fitting strategies such as traditional system identification and physics-informed neural networks that identify or refine parameter values within first-principles models. White-box methods (i.e., first-principles models) are suitable when first-principles models and their corresponding parameter values are available for systems with well-understood physicochemical phenomena. Gray-box methods are applicable when the functional forms of first-principles models are available with unknown parameter values. Black-box methods are used for modeling complex nonlinear systems with limited knowledge of their first-principles models. In this work, we focus on addressing the challenges associated with black-box methods, particularly improving online computational efficiency in neural network-based optimization and control.
\end{remark}

\section{Conclusion}
\label{sec5}
In this study, we developed a novel neural network architecture (i.e., IC-LSTM) that ensures convexity of the output with respect to the input, specifically tailored for convex neural network-based optimization and control. In particular, our framework excels in terms of computational efficiency in solving neural network-based optimization for real-time operations. Through the real-time optimization of a real-world hybrid energy system at LHT Holdings and the simulation study of a dynamic CSTR system, we demonstrated the efficacy and efficiency of our proposed framework. This work serves as a pivotal bridge between ICNNs and their applications within the realm of a variety of engineering systems, such as energy and chemical applications.

\section{Acknowledgments}
Financial support from the A*STAR MTC YIRG 2022 Grant (222K3024) and  NRF-CRP Grant 27-2021-0001 is gratefully acknowledged.

\newpage
\bibliography{reference}

\end{document}